\documentclass[final]{alt2023} % Include author names
\title[Efficient Global Planning in Large MDPs via Stochastic Primal-Dual Optimization]{Efficient Global Planning in Large MDPs \\ via Stochastic Primal-Dual Optimization}
\usepackage{times}

\usepackage{enumitem}
\usepackage{amsmath}
\usepackage{amssymb}
\usepackage{mathtools}

\RequirePackage{algorithm}
\RequirePackage{algorithmic}

\newtheorem{assumption}{Assumption}

\altauthor{%
 \Name{Gergely Neu} \Email{gergely.neu@gmail.com}\\
 \addr Universitat Pompeu Fabra, Barcelona, Spain
 \AND
 \Name{Nneka Okolo} \Email{nnekamaureen.okolo@upf.edu}\\
 \addr Universitat Pompeu Fabra, Barcelona, Spain%
}

\newcommand{\piout}{\pi_{\text{out}}}
\newcommand{\regret}{\mathfrak{R}}
\newcommand{\primalregret}{\regret^{\textup{p}}}
\newcommand{\dualregret}{\regret^{\textup{d}}}
\newcommand{\coredelta}{\Delta_{\textup{core}}}
\newcommand{\coreeps}{\varepsilon_{\textup{core}}}
\newcommand{\epsopt}{\varepsilon_{\textup{opt}}}
\newcommand{\epsapprox}{\varepsilon_{\textup{approx}}}

\newcommand{\Mpi}{M^\pi}
\newcommand{\Mstar}{M^*}

\newcommand{\thbound}{D_\gamma}

 \newcommand{\X}{\mathcal{X}}
  
\newcommand{\F}{\mathcal{F}}
\newcommand{\A}{\mathcal{A}}
\newcommand{\B}{\mathcal{B}}
\newcommand{\U}{\mathcal{U}}

\newcommand{\M}{\mathcal{M}}

\newcommand{\Rn}[0]{\mathbb{R}} % real numbers
\renewcommand{\vec}[1]{{\boldsymbol{{#1}}}} % vector
\newcommand{\Z}{\mathcal{Z}}

\newcommand{\bb}[2]{\mathbb{B}(\thbound)}
\newcommand{\real}{\mathbb{R}}

\newcommand{\LL}{\mathcal{L}}
\newcommand{\GG}{\mathcal{G}}

\newcommand{\DD}[2]{\mathcal{D}\pa{#1\middle\|#2}}
\newcommand{\HH}[2]{\mathcal{H}\pa{#1\middle\|#2}}

\newcommand{\DDKL}[2]{\mathcal{D}\pa{#1\middle\|#2}}

\newcommand{\OO}{\mathcal{O}}

\newcommand{\II}[1]{\mathbb{I}_{\left\{#1\right\}}}

\newcommand{\EE}[1]{\mathbb{E}\left[#1\right]}
\newcommand{\EEpi}[1]{\mathbb{E}_\pi\left[#1\right]}
\newcommand{\EEcpi}[2]{\mathbb{E}_\pi\left[#1\middle|#2\right]}

\newcommand{\EEs}[2]{\mathbb{E}_{#2}\left[#1\right]}

\newcommand{\EEt}[1]{\mathbb{E}_t\left[#1\right]}
\newcommand{\EET}[1]{\mathbb{E}_T\left[#1\right]}
\newcommand{\EEti}[1]{\mathbb{E}_{t,i}\left[#1\right]}

\newcommand{\EEc}[2]{\mathbb{E}\left[#1\left|#2\right.\right]}
\newcommand{\EEcc}[2]{\mathbb{E}\left[\left.#1\right|#2\right]}

\def\argmin{\mathop{\textup{ arg\,min}}}
\def\argmax{\mathop{\textup{ arg\,max}}}
\def\supp{\mathop{\textup{supp}}}
\def\rank{\mathop{\textup{rank}}}
\newcommand{\IBE}{\textup{IBE}}
\newcommand{\BE}{\textup{BE}}
\newcommand{\ra}{\rightarrow}

\newcommand{\iprod}[2]{\left\langle#1,#2\right\rangle}

\newcommand{\biprod}[2]{\bigl\langle#1,#2\bigr\rangle}

\newcommand{\norm}[1]{\left\|#1\right\|}

\newcommand{\twonorm}[1]{\norm{#1}_2}
\newcommand{\infnorm}[1]{\norm{#1}_\infty}

\newcommand{\ev}[1]{\left\{#1\right\}}
\newcommand{\abs}[1]{\left|#1\right|}

\newcommand{\pa}[1]{\left(#1\right)}
\newcommand{\bpa}[1]{\bigl(#1\bigr)}

\newcommand{\wt}{\widetilde}

\newcommand{\ol}[1]{\overline{#1}}

\newcommand{\transpose}{^\mathsf{\scriptscriptstyle T}}

\definecolor{darkpastelgreen}{rgb}{0.01, 0.75, 0.24}

\usepackage[normalem]{ulem}

\begin{document}

\maketitle

\begin{abstract}%
  We propose a new stochastic primal-dual optimization algorithm for planning in a large discounted Markov decision process with a generative model and linear function approximation. Assuming that the feature map approximately satisfies standard realizability and Bellman-closedness conditions and also that the feature vectors of all state-action pairs are representable as convex combinations of a small core set of state-action pairs, we show that our method outputs a near-optimal policy after a polynomial number of queries to the generative model. Our method is computationally efficient and comes with the major advantage that it outputs a single softmax policy that is compactly represented by a low-dimensional parameter vector, and does not need to execute computationally expensive local planning subroutines in runtime.%
\end{abstract}


\section{Introduction}
Finding near-optimal policies in large Markov decision processes (MDPs) is one of the most important tasks encountered in model-based reinforcement learning \citet{sutton2018reinforcement}. 
This problem (commonly referred to as \emph{planning}) presents both computational and statistical challenges when having access only to a generative model of the environment: an efficient planning method should use little computation and few samples drawn from the model. While the complexity of planning in small Markov decision processes is already well-understood by now (cf.~\citealp{azar2013minimax,sidford2018near,agarwal2020model}), extending the insights to large state spaces with function approximation has remained challenging. One particular challenge that remains unaddressed in the present literature is going beyond \emph{local planning} methods that require costly online computations for producing each action during execution time. In this work, we advance the state of the art in planning for large MDPs by designing and analyzing a new \emph{global planning} algorithm that outputs a simple, compactly represented policy after performing all of its computations in an offline fashion.

Our method is rooted in the classic linear programming framework for sequential decision making first proposed by \citet{manne1960linear,d1963probabilistic,denardo1970linear}. 
This approach formulates the problem of finding an optimal behavior policy as a linear program (LP) with a number of variables and constraints that is proportional to the size of the state-action space of the MDP. Over the last several decades, numerous attempts have been made to derive computationally tractable algorithms from the LP framework, most notably via reducing the number of variables using function approximation techniques. This idea has been first explored by \citet{schweitzer1985generalized}, whose ideas were further developed and popularized by the seminal work of \citet{de2003linear}. Further approximations were later proposed by \citet{de2004constraint,petrik2009constraint,lakshminarayanan2017linearly}, whose main focus was reducing the complexity of the LPs while keeping the quality of optimal solutions under control.

More recently, the LP framework has started to inspire practical methods for reinforcement learning and, more specifically, planning with a generative model. The most relevant to the present paper is the line of work initiated by \citet{wang2016online,wang2017primal}, who proposed planning algorithms based on primal-dual optimization of the Lagrangian associated with the classic LP. Over the last several years, this computational recipe has been perfected for small MDPs by works like \citet{cheng2020reduction,jin2020efficiently,tiapkin2022primal}, and even some extensions using linear function approximation have been proposed by \citet{chen2018scalable,serrano2020faster}. While these methods have successfully reduced the number of primal and dual variables, they all require stringent conditions on the function approximator being used and their overall sample complexity scales poorly with the size of the MDP. For instance, the bounds of \citet{chen2018scalable} scale with a so-called ``concentrability coefficient'' that can be as large as the size of the state space, thus failing to yield meaningful guarantees for large MDPs. Furthermore, these methods require parametrizing the space of so-called state-action occupancy measures, which is undesirable given the complexity of said space. 

In the present work, we develop a new stochastic primal-dual method based on a relaxed version of the classical LP. This relaxed LP (inspired by the works of \citealp{mehta2009q,NPB20} and \citealp{serrano2021logistic}) features a linearly parametrized action-value function and guarantees to produce optimal policies as solutions under well-studied conditions like the linear MDP assumption of \citet{yang2019sample,JYWJ20}. Our method iteratively updates the primal and dual variables via a combination of stochastic mirror descent steps and a set of implicit update rules tailor-made for the relaxed LP. Under a so-called \emph{core state-action-pair} assumption, we show that the method produces a near-optimal policy with sample and computation complexity that is polynomial in the relevant problem parameters: the size of the core set, the dimensionality of the feature map, the effective horizon, and the desired accuracy. Additional assumptions required by our analysis are the near-realizability of the Q-functions and closedness under the Bellman operators of all policies. The main merit of our algorithm is that it produces a compactly represented softmax policy which requires no access to the generative model in runtime.

The works most directly comparable to ours are \citet{wang2021sample}, \citet{shariff2020efficient}, and \citet{yin2022efficient}---their common feature being the use of a core set of states or state-action pairs for planning. \citet{wang2021sample} provide a sample complexity bound for finding a near-optimal policy in linear MDPs, but their algorithm is computationally intractable\footnote{At least we are not aware of a computationally efficient global planning method that works for the discounted case they consider.} due to featuring a planning subroutine in a linear MDP. \citet{shariff2020efficient} and \citet{yin2022efficient} both propose local planning methods that require extensive computation in runtime, but on the other hand some of their assumptions are weaker than ours. In particular, they only require realizability of the value functions but continue to operate without the function class being closed under all Bellman operators. For a detailed discussion on the role and necessity of such assumptions, we refer to the excellent discussion provided by \citet{yin2022efficient} on the subject.

\paragraph{Notation.}
In the $n$-dimensional Euclidean space $\real^n$, we denote the vector of all ones by $\mathbf{1}$, the zero vector by $\mathbf{0}$ and the $i$-th coordinate vector by $\vec{e_{i}}$. For, vectors $a,b\in \Rn^{m}$, we use $a\leq b$ to denote elementwise comparison, meaning that $a_i \le b_i$ is satisfied for all entries $i$. For any finite set $D$, $\Delta_{D} = \{p\in\Rn^{|D|}_{+}|\|p\|_{1} = 1\}$ denotes the set of all probability distributions over its entries. We define the relative entropy between two distributions $p,p'\in\Delta_{D}$ as $\DDKL{p}{p'}=\sum_{i=1}^{\abs{D}}p_{i}\log\frac{p_{i}}{p_{i}'}$.
In the context of iterative algorithms, we will use the notation $\F_{t-1}$ to refer to the sigma-algebra generated by all events up to the end of iteration $t-1$, and use the shorthand notation $\EEt{\cdot} = \EEcc{\cdot}{\F_{t-1}}$ to denote expectation conditional on the past history encountered by the method.

\section{Preliminaries}
\label{prelim}
We study a discounted Markov Decision Processes \citep{puterman2014markov} denoted by the quintuple $(\X,\A,r,P,\gamma)$ with $\X$ and 
$\A$ representing finite (but potentially very large) state and action spaces of cardinality $X = |\X|, A = |\A|$ 
respectively. The reward function is denoted by $r:\X\times \A\rightarrow \Rn$, and the transition function by $P:\X\times \A\rightarrow \Delta_{\X}$. 
We will often represent the reward function by a vector in $\Rn^{XA}$ and the transition function by the operator $P\in\Rn^{XA\times X}$ which acts on functions $v\in\Rn^X$ by assigning $(P v)(x,a) = \sum_{x'} P(x'|x,a) v(x')$ for all $x,a$. Its adjoint $P\transpose \in\Rn^{ X\times XA}$ is similarly defined on functions $u\in\Rn^{XA}$ via the assignment $(P\transpose u)(x) = \sum_{x',a'} P(x|x',a') u(x',a')$ for all $x$. We also define the operator $E\in\Rn^{XA\times X}$ and its adjoint $E\transpose \in\Rn^{X\times XA}$ acting on respective vectors $v\in\Rn^{X}$ and $u\in\Rn^{XA}$ through the assignments $(E v)(x,a) = v(x)$ and $(E\transpose u)(x) = \sum_{a}u(x,a)$. For simplicity, we assume the rewards are bounded in $[0,1]$ and let $\Z = \{(x,a)|x\in \X, a\in \A\}$ denote the set of all possible state action pairs with cardinality $Z = |\Z|$ to be used when necessary.

The Markov decision process describes a sequential decision-making process where in each round $t=0,1,\dots$, the \emph{agent} observes the state of the environment $x_t$, takes an action $a_t$, and earns a potentially random reward with expectation $r(x_t,a_t)$. The state of the environment in the next round $t+1$ is generated randomly according to the transition dynamics as $x_{t+1}\sim P(\cdot|x_t,a_t)$. The initial state of the process $x_0$ is drawn from a fixed distribution $\nu_0\in\Delta_{\X}$. In the  setting we consider, the agent's goal is to maximize its normalized discounted return $(1-\gamma)\EE{\sum_{t=0}^\infty \gamma^t r(x_t,a_t)}$, where $\gamma \in (0,1)$ is the discount factor, and the expectation is taken over the random transitions generated by the environment, the random initial state, and the potential randomness injected by the agent. It is well-known that maximizing the discounted return can be achieved by following a memoryless time-independent decision making rule mapping states to actions. Thus, we restrict our attention to \emph{stationary stochastic policies} $\pi:\X\ra\Delta_{\A}$ with $\pi(a|x)$ denoting the probability of the policy taking action $a$ in state $x$. We define the mean operator $\Mpi:\real^{XA}\ra\real^X$ with respect to policy $\pi$ that acts on functions $Q\in\Rn^{X\times A}$ as $(\Mpi Q)(x) = \sum_{a} \pi(a|x) Q(x,a)$ and the 
(non-linear) max operator $M:\Rn^{XA}\rightarrow\Rn^{X}$ that acts as $(\Mstar Q)(x) = \max_{a\in\A}Q(x,a)$.
The value function and action-value function of policy $\pi$ are respectively defined as
$V^{\pi}(x) = \EEcpi{\sum_{t=0}^\infty \gamma^t r(x_t,a_t)}{x_0 = x}$ and $Q^{\pi}(x,a) = \EEcpi{\sum_{t=0}^\infty \gamma^t r(x_t,a_t)}{x_0 = x, a_0 = a}$, where the notation $\EEpi{\cdot}$ signifies that each action is selected by following the policy as $a_t\sim \pi(\cdot|x_t)$. The value functions of a policy $\pi$ are known to satisfy the Bellman equations \citep{bellman1966dynamic} and the value functions of an optimal policy $\pi^*$ satisfy the Bellman optimality equations, which can be conveniently written in our notation as 
\begin{equation}\label{eq:BE}
 Q^\pi = r + \gamma PV^\pi \qquad\qquad\mbox{and}\qquad\qquad Q^* = r + \gamma PV^*,
\end{equation}
with $V^\pi = \Mpi Q^\pi$ and $V^* = \Mstar Q^*$.
Any optimal policy satisfies $\supp(\pi(\cdot|x)) \subseteq \argmax_a Q^*(x,a)$. We refer the reader to \citet{puterman2014markov} for the standard proofs of these fundamental results.

Our approach taken in this paper will primarily make use of an alternative formulation of the MDP optimization based on linear programming, due to \citet{manne1960linear} (see also \citealp{d1963probabilistic}, \citealp{denardo1970linear}, and Section~6.9 in \citealp{puterman2014markov}). This formulation phrases the optimal control problem as a search for an \emph{occupancy measure} with maximal return. The state-action occupancy measure of a policy $\pi$ is defined as
\[
 \mu^\pi(x,a) = (1-\gamma)\EEpi{\sum_{t=0}^\infty \gamma^t \II{(x_t,a_t)=(x,a)}},
\]
which allows rewriting the expected normalized return of $\pi$ as $R_\gamma^\pi = \iprod{\mu^\pi}{r}$. The corresponding state-occupancy measure is defined as the state distribution $\nu^\pi(x) = \sum_{a} \mu^\pi(x,a)$. Denoting the direct product of a state distribution $\nu$ and the policy $\pi$ by $\nu\circ\pi$ with its entries being $(\nu\circ\pi)(x,a) = \nu(x)\pi(a|x)$, we can notice that the state-action occupancy measure induced by $\pi$ satisfies $\mu^\pi = \nu^\pi\circ\pi$.
The set of all occupancy measures can be fully characterized by a set of linear constraints, which allows rewriting the optimal control problem as the following linear program (LP):
\begin{equation}
\begin{split}\label{eq:LP_primal}
	\text{max}_{\mu\in\real_+^{XA}}  &\,\,\,\, \langle \mu , r\rangle \\
	\text{subject to}&\,\,\,\, E\transpose \mu=  (1 - \gamma)\nu_{0} + \gamma P\transpose \mu.
\end{split}
\end{equation}
Any optimal solution $\mu^*$ of this LP can be shown to correspond to the occupancy measure of an optimal policy $\pi^*$, and generally policies can be extracted from feasible points $\mu$ via the rule $\pi_\mu(a|x) = \mu(x,a) / \sum_{a'} \mu(x,a')$ (subject to the denominator being nonzero). The dual of this linear program takes the following form:
\begin{equation}
\begin{split}\label{eq:LP_primal}
	\text{min}_{V\in\real^X}  &\,\,\,\, (1-\gamma)\langle \nu_0 , V\rangle \\
	\text{subject to}&\,\,\,\, EV \ge  r + \gamma PV.
\end{split}
\end{equation}
The optimal value function $V^*$ is known to be an optimal solution for this LP, and is the unique optimal solution provided that $\nu_0$ has full support over the state space.

The above linear programs are not directly suitable as solution tools for large MDPs, given that they have a large number of variables and constraints. Numerous adjustments have been proposed over the last decades to address this issue \citep{schweitzer1985generalized,de2003linear,de2004constraint,lakshminarayanan2017linearly,serrano2020faster}. One recurring theme in these alternative formulations is relaxing some of the constraints by the introduction of a \emph{feature map}. In this work, we use as starting point an alternative proposed by \citet{serrano2021logistic}, who use a feature map $\varphi: \X\times\A \ra \real^d$ represented by the $XA\times d$ \emph{feature matrix} $\Phi$, and rephrase the original optimization problem as the following LP:
\begin{align}\label{eq:Q-ALP} 
		\nonumber\text{max}_{\mu,u\in\real^{XA}_{+}}  &\,\,\,\, \langle \mu\,, r\rangle \\
		\nonumber \text{subject to}&\,\,\,\, E\transpose u=  (1 - \gamma)\nu_{0} + \gamma P\transpose \mu,\\
		\,\, & \,\,\,\, \Phi\transpose \mu = \Phi\transpose u~.
\end{align}
The dual of this LP is given as
\begin{align*}
	\nonumber\text{min}_{\theta\in\real^d,V\in\real^{X}}  &\,\, (1-\gamma)\langle \nu_{0}\,, V\rangle \\
	\nonumber \text{subject to}&\,\,\,\, \Phi\theta \ge  r + \gamma PV, \\
	\,\, & \,\,\,\, EV \ge \Phi\theta~.
\end{align*}
As shown by \citet{serrano2021logistic}, optimal solutions of the above relaxed LPs correspond to optimal occupancy measures and value functions under the popular linear MDP assumption of \citet{yang2019sample,JYWJ20}. 
The key merit of these LPs is that the dual features a linearly parametrized action-value function $Q_\theta = \Phi\theta$, which allows the extraction of a simple greedy policy $\pi_\theta(a|x) = \II{a = \argmax_{a'} Q_\theta(x,a')}$ from any dual solution $\theta$, and in particular an optimal policy can be extracted from its optimal solution.

\section{A tractable linear program for large MDPs with linear function approximation}\label{sec:RALP}
The downside of the classical linear programs defined in the previous sections is that they all feature a large number of variables and constraints, even after successive relaxations. In what follows, we offer a new version of the LP~\eqref{eq:Q-ALP} that gets around this issue and allows the development of a tractable primal-dual algorithm with strong performance guarantees. In particular, we will reduce the number of variables in the primal LP~\eqref{eq:Q-ALP} by considering only sparsely supported state-action distributions instead of full occupancy measures, which will be justified by the following technical assumption made on the MDP structure:
\begin{assumption}{(Approximate Core State-Action Assumption)}\label{ass:Core_SA_approx}
	The feature vector of any state-action pair $(x,a)\in\Z$ can be approximately expressed as a convex combination of features evaluated at a set of $m$ \emph{core state-action pairs} $(x',a')\in \Tilde{\Z}$, up to an error $\coredelta\in\real^{XA\times d}$. That is, for each $(x,a)\in\Z$, there exists a set of coefficients satisfying $b(x',a'|x,a)\geq 0$  and $\sum_{x',a'} b(x',a'|x,a)=1$ such that $\varphi(x,a) = \sum_{x',a'} b(x',a'|x,a)\varphi(x',a') + \coredelta(x,a)$. Furthermore, for every $x,a$, the misspecification error satisfies $\norm{\coredelta(x,a)}_2 = \coreeps(x,a)$.
\end{assumption}

It will be useful to rephrase this assumption using the following handy notation. Let $\U \in\Rn^{m\times Z}_{+}$ denote a selection matrix such that, $\Tilde{\Phi} = \U \Phi\in\Rn^{m\times d}$ is the core feature matrix with rows corresponding to $\Phi$ evaluated at core state-action pairs. Furthermore, the interpolation coefficients from Assumption~\ref{ass:Core_SA_approx} can be organized into a stochastic matrix $\B\in\Rn^{Z\times m}$ with $\B(x,a) = \{b(x',a'|x,a)\}_{(x',a')\in\Tilde{\Z}}\in\Rn^{m}_{+}$ for $(x,a)\in\Z$. Then, the assumption can be simply rephrased as requiring the condition that $\Phi = \B \U \Phi +\coredelta$. Note that both $\U$ and $\B$ are stochastic matrices satisfying $\U \mathbf{1} = \mathbf{1}$ and $\B \mathbf{1} = \mathbf{1}$, and the same holds for their product $\B \U \mathbf{1} = \mathbf{1}$. Note however that $\B \U$ is a rank-$m$ matrix, which implies that the assumption can only be satisfied with zero error whenever $m\ge \rank(\Phi)$, which in general can be as large as the feature dimensionality $d$. Whether or not it is possible to find a set of core-state-action pairs in a given MDP is a nontrivial question that we discuss in Section~\ref{sec:conc}.

With the notation introduced above, we are ready to state our relaxation of the LP~\eqref{eq:Q-ALP} that serves as the foundation for our algorithm design:
\begin{equation}
\begin{split}\label{eq:QRALP-primal}
	\text{max}_{\lambda\in\real_+^{m}, u\in\real^{XA}_+}  &\,\, \langle \lambda\,,\U r\rangle\\
	\text{subject to}&\,\, E\transpose u=  (1 - \gamma)\nu_{0} + \gamma P\transpose \U\transpose\lambda,\\
	\,\, & \,\, \Phi\transpose \U\transpose\lambda= \Phi\transpose u~.
\end{split}
\end{equation}
The dual of the LP can be written as follows:
\begin{equation}
\begin{split}\label{eq:QRALP-dual}
    \text{min}_{\theta\in\real^d,V\in\real^X}
	& \,\, (1 - \gamma)\langle \nu_{0}\,,V\rangle\\
	\text{subject to}& \,\, EV \geq Q_{\theta}\,,\\
	& \,\, \U Q_{\theta} \geq \U (r + \gamma PV)~. 
\end{split}
\end{equation}

The above LPs can be shown to yield optimal solutions that correspond to optimal occupancy measures and value functions under a variety of conditions. The first of these is the so-called linear MDP condition due to \citet{yang2019sample,JYWJ20}, which is recalled below:
\begin{definition}{(Linear MDP)}\label{def:linMDP}
	An MDP is called a linear MDP if there exists $W\in\Rn^{d\times X}$ and $\vartheta\in\Rn^{d}$ 
	such that, the transition matrix $P$ and reward vector $r$ can be written as the linear functions $P = \Phi W$ and $r = \Phi \vartheta$.
\end{definition}
It is easy to see that the relaxed LPs above retain the optimal solutions of the original LP as long as $\coredelta
=0$---we provide the straightforward proof in Appendix~\ref{app:linMDP-realizability}.  As can be seen from the Bellman equations~\eqref{eq:BE}, the action-value function of any policy $\pi$ can be written as $Q^\pi = \Phi\theta^\pi$ for some $\theta^\pi$ under the linear MDP assumption. The linearity of the transition function is a very strong condition that is often not satisfied in problems of practical interest, which motivates us to study a more general class of MDPs with weaker feature maps. The following two concepts will be useful in characterizing the power of the feature map.

Roughly speaking, the conditions below correspond to supposing that all Q-functions can be \emph{approximately} represented by some parameter vectors with bounded norms, and similarly the Bellman operators applied to feasible Q-functions are also approximately representable with linear functions with bounded coefficients. Concretely, we will assume that all feature vectors satisfy $\twonorm{\varphi(x,a)} \leq R$ for all $(x,a)$ and will work with parameter vectors with norm bounded as $\twonorm{\theta}\le \thbound$, whose set will be denoted as $\bb{d}{\thbound} = \ev{\theta\in\real^d: \twonorm{\theta}\le\thbound}$. The first property of interest we define here characterizes the best approximation error that this set of parameters and features has in terms of representing the action-value functions:
\begin{definition}{(Q-Approximation Error)}\label{def:QAE}
The Q-approximation error  associated with a policy $\pi$ is defined as $\epsilon_\pi = \inf_{\theta\in \bb{b}{\thbound}} \infnorm{Q^\pi - \Phi\theta}$.
\end{definition}
Since the true action-value functions of any policy satisfy $\infnorm{Q^\pi} \le \frac{1}{1-\gamma}$ and members of our function class satisfy $\infnorm{Q_\theta} \le R\thbound$, the bound $\thbound$ should scale linearly with $\frac{1}{1-\gamma}$ to accommodate all potential Q-functions with small error. 
The second property of interest is the ability of the feature map to capture applications of the Bellman operator to functions within the approximating class:
\begin{definition}{(Inherent Bellman Error)}\label{def:IBE}
 The \emph{Bellman Error} (BE) associated with the pair of parameter vectors $\theta,\theta'\in\bb{d}{\thbound}$ and a policy $\pi$ is defined as the state-action vector $\BE^{\pi}(\theta,\theta')\in\real^{XA}$ with components
 \[
 \BE^{\pi}(\theta,\theta') = r + \gamma P\Mpi Q_{\theta'} - Q_\theta.
 \]
 The \emph{Inherent Bellman Error} is then defined as $\IBE = \sup_{\pi}\sup_{\theta'\in\bb{d}{\thbound}}\inf_{\theta\in\bb{d}{\thbound}} \bigl\|\BE^{\pi}(\theta,\theta')\bigr\|_\infty$.
\end{definition}
Function approximators with zero IBE are often called ``Bellman complete'' and have been intensely studied in 
reinforcement learning \citep{ASM08,chen2019information,zanette2020learning}---note however that our notion is stronger 
than some others appearing in these works as it requires small error for all policies $\pi$ and not just greedy ones. As 
shown by \citet{JYWJ20}, linear MDP models enjoy zero inherent Bellman error in the stricter sense used in the above 
definition. They also establish that the converse is also true: having zero IBE for all policies implies linearity of 
the transition model. Provided that $\thbound$ is set large enough, one can show that feature maps with zero IBE 
satisfy $\varepsilon_{\pi} = 0$ for all policies $\pi$, and optimal solutions of the above relaxed LPs yield optimal 
policies when $\coredelta = 0$. The interested reader can verify this by appropriately adjusting the arguments in 
Appendix~\ref{app:linMDP-realizability} (see Footnote~\ref{fn:realizability}).

\section{Algorithm and main results}
We now turn to presenting our main contribution: a computationally efficient planning algorithm to approximately solve the LPs~\eqref{eq:QRALP-primal} and~\eqref{eq:QRALP-dual} via stochastic primal-dual optimization. Throughout this section, we will assume sampling access to $\nu_0$ and a generative model of the MDP that can produce i.i.d.~samples from $P(\cdot|x,a)$ for any of the core state-action pairs.

We first introduce the \emph{Lagrangian function} associated with the two LPs:
\begin{equation}
	\LL(\lambda, u; \theta, V)
	= \langle\lambda\,, \U (r + \gamma PV - Q_{\theta})\rangle + (1 - \gamma)\langle \nu_{0}\,,V\rangle + \langle 
u\,,Q_{\theta} - EV \rangle.
\end{equation}
To facilitate optimization, we will restrict the decision variables to some naturally chosen compact sets. For the primal variables, we require that $\lambda\in\Delta_{\Tilde{\Z}}$ and $u\in\Delta_{\Z}$. For $\theta$, we restrict our attention to the domain $\bb{d}{\thbound} = \ev{\theta\in\real^d: \twonorm{\theta}\le\thbound}$ and we will consider only $V\in S = \{V\in\Rn^{X}~|~\|V\|_{\infty}\leq R\thbound\}$. Hence, we seek to solve 
\begin{equation*}
	\min_{\theta\in\bb{d}{\thbound}, V\in S}\max_{\lambda\in\Delta_{\Tilde{\Z}} , u\in \Delta_{\Z}} \LL(\lambda, u; 
\theta, V).
\end{equation*}

Our algorithm is inspired by the classic stochastic optimization recipe of running two regret-minimization algorithms for optimizing the primal and dual variables, and in particular using two instances of mirror descent \citep{NY83,BT03} to update the low-dimensional decision variables $\theta$ and $\lambda$. There are, however, some significant changes to the basic recipe that are made necessary by the specifics of the problem that we consider. 
The first major challenge that we need to overcome is that the variables $u$ and $V$ are high-dimensional, so running mirror descent on them would result in an excessively costly algorithm with runtime scaling linearly with the size of the state space. To avoid this computational burden, we design a special-purpose implicit update rule for these variables that allow an efficient implementation without having to loop over the entire state space. The second challenge is somewhat more subtle: the implicit updates employed to calculate $u$ make it impossible to adapt the standard method of extracting a policy from the solution \citep{wang2017primal,cheng2020reduction,jin2020efficiently}, which necessitates an alternative policy extraction method. This in turn requires more careful ``policy evaluation'' steps in the implementation, which is technically achieved by performing several dual updates on $\theta$ between each primal update to $\lambda$ and $u$. 
Finally, we need to make sure that the subsequent policies calculated by the algorithm do not change too rapidly, which is addressed by using a softmax policy update.

More formally, our algorithm performs the following steps in each iteration $t=1,2,\dots,T$:
\begin{enumerate}[parsep=1pt,itemsep=1pt]
 \item set $\nu_{t} = \gamma P\transpose\U\transpose\lambda_{t} + (1-\gamma)\nu_0$,
 \item set $u_{t} = \nu_{t} \circ \pi_{t}$,
 \item perform $K$ stochastic gradient descent updates starting from $\theta_{t-1}$ on 
$\LL(\lambda_{t},u_{t};\cdot,V_{t-1})$ and set $\theta_{t}$ as the average of the iterates,
 \item update the action-value function as $Q_{t} = \Phi\theta_{t}$,
 \item update the state-value function as $V_{t}(x) = \sum_{a} \pi_{t}(a|x) Q_{t}(x,a)$.
 \item perform a stochastic mirror ascent update starting from $\lambda_{t}$ on $\LL(\cdot,u_t;\theta_{t},V_{t})$ to obtain $\lambda_{t+1}$,
 \item update $\pi_{t+1}$ as $\pi_{t+1}(a|x) \propto \pi_t(a|x)e^{\beta Q_{t}(x,a)}$,
\end{enumerate}
We highlight that several of the above abstractly defined steps are only performed implicitly by the algorithm. First, the state distribution $\nu_{t}$ is never actually calculated by the algorithm, as it is only needed to generate samples for the computation of stochastic gradients with respect to $\theta$. Note that $\nu_{t}$ is chosen to satisfy the primal constraint on $u_{t}$ exactly. Second, the value functions $V_{t}$ do not have to be calculated for all states, only for the ones being accessed by the stochastic gradient updates for $\lambda$. Similarly, the policy $\pi_t$ does not need to be computed for all states, but only locally wherever necessary.
The details of the stochastic updates are clarified in the pseudocode provided as Algorithm~\ref{alg:Main}; we only note here that both the primal and dual updates can be implemented efficiently via a single access of the generative model per step, making for a total of $K+1$ samples per each iteration in the outer loop. Thus, the total number of times that the algorithm queries the generative model is $T(K+1)$. The algorithm finally returns a policy $\pi_J$ with the index $J$ selected uniformly at random. This policy can be written as $\pi_J(a|x) \propto \pi_1(a|x) e^{\beta \sum_{t=1}^{J-1}Q_t(x,a)}$, where the exponent can be compactly represented by the $d$-dimensional parameter vector $\Theta_J = \sum_{t=1}^{J-1} \theta_t$.

\begin{algorithm}[H]
	\caption{Global planning via primal-dual stochastic optimization.}
	\label{alg:Main}
	\begin{algorithmic}
		\STATE {\bfseries Input:} Core set $\Tilde{\Z}$, learning rates $\eta$, $\beta$, $\alpha$, initial 
iterates $\theta_{0}\in\bb{d}{\thbound}$, $\lambda_1\in\Delta_{\wt{Z}}$, $\pi_1\in\Pi$.
		\FOR{$t=1$ {\bfseries to} $T$}
		\STATE \emph{Stochastic gradient descent}:
		\STATE Initialize: $\theta_{t}^{(1)} = \theta_{t-1}$;
		\FOR{$i=1$ {\bfseries to} $K$}
		\STATE Sample $x_{0,t}^{(i)}\sim\nu_{0}$ and $a_{0,t}^{(i)}\sim\pi_{t}(\cdot|x_{0,t}^{(i)}),$\\
		\hspace{3.5em}$(x_{t}^{(i)}, a_{t}^{(i)})\sim \lambda_{t},$\\
		\hspace{3.5em}$\ol{x}_{t}^{(i)}\sim P(\cdot|x_{t}^{(i)}, 
a_{t}^{(i)})$ and $\ol{a}_{t}^{(i)}\sim\pi_{t}(\cdot|\ol{x}_{t}^{(i)});$ 
		\STATE Compute $\Tilde{g}_{\theta}(t,i) = (1-\gamma)\varphi(x_{0,t}^{(i)},a_{0,t}^{(i)}) + \gamma\varphi(\ol{x}_{t}^{(i)},\ol{a}_{t}^{(i)})- \varphi(x_{t}^{(i)},a_{t}^{(i)})$;
		\STATE Update $\theta_{t}^{(i+1)} = \Pi_{\bb{d}{\thbound}}(\theta_{t}^{(i)} - \alpha \Tilde{g}_{\theta}(t,i))$;
		\ENDFOR
		\STATE Compute $\theta_{t} = \dfrac{1}{K}\sum_{i=1}^{K}\theta_{t}^{(i)}$;
		\STATE \emph{Stochastic mirror ascent}:
		\STATE Sample $(x_{t}, a_{t})\sim \text{Unif}(\Tilde{\Z}), (r_{t}, y_{t})\sim P(\cdot|x_{t}, a_{t})$;
		\STATE Compute $V_{t}(y_{t}) = \sum_{a} \pi_{t}(a|y_{t}) Q_{t}(y_{t},a)$;
		\STATE Compute $\Tilde{g}_{\lambda}(t) = m[r(x_{t},a_{t})  + \gamma V_{t}(y_{t}) - Q_{t}(x_{t},a_{t})]\vec{e}_{(x_{t},a_{t})}$;
		\STATE Update $\lambda_{t+1} = \lambda_{t}e^{\eta\Tilde{g}_{\lambda}(t)} / 
\iprod{\lambda_{t}e^{\eta\Tilde{g}_{\lambda}(t)}}{\vec{1}}$;
			\STATE \emph{Policy update}:
			\STATE Compute $\pi_{t+1} = \text{softmax}\pa{\beta \sum_{k=1}^{t} \Phi\theta_k}$.
		\ENDFOR
		\STATE {\bfseries Return:} $\pi_{J}$ with $J\sim \text{Unif}\{1,\cdots,T\}$.
	\end{algorithmic}
\end{algorithm} 
Our main result regarding the performance the algorithm is the following.
\begin{theorem}\label{thm:main}
Suppose that Assumption~\ref{ass:Core_SA_approx} holds and that the initial policy $\pi_1$ is uniform over the actions. Then, after $T$ iterations Algorithm~\ref{alg:Main} outputs a policy $\piout$ satisfying 
\[
 \EE{\iprod{\mu^* - \mu^{\piout}}{r}} \le \epsopt + \epsapprox.
\]
Here, $\epsopt$ is a bound on the expected optimization error defined as
\[
	\epsopt =  \frac{\DD{\lambda^*}{\lambda_1}}{\eta T} + \frac{\log |\A|}{\beta T} + \frac{2 \thbound^2}{\alpha K} + \frac{\eta m^2(1 + 2R\thbound)^2}{2} + \frac{\beta R^2\thbound^2}{2} + 2 \alpha R^2,
\]
and $\epsapprox$ is a bound on the expected approximation error defined as
\[
 \epsapprox = 2\EE{\varepsilon_{\piout}} + 2\IBE + 2\thbound\iprod{\mu^*}{\coreeps}.
\]
In particular, for any target accuracy $\varepsilon>0$, setting $K=T/\pa{m^2 \log\pa{m|\A|}}$ and tuning the hyperparameters appropriately, the expected optimization error satisfies $\epsopt \le \varepsilon$ after $n_\varepsilon$ queries to the generative model with
\[
 n_\varepsilon = \OO\pa{\frac{m^2 R^4 \thbound^4 \log(m|\A|)}{\varepsilon^4}}.
\]
\end{theorem}

A few comments are in order. First, recall that the Q-approximation error $\varepsilon_{\piout}$ and the inherent Bellman error terms are zero for linear MDPs, where the only remaining approximation error term corresponds to the extent of violation of the core state-action assumption. Interestingly, our result shows that this approximation error does not need to be uniformly small, but only needs to be under control in the states that the optimal policy visits. Thus, provided access to good core state-action pairs, our algorithm is guaranteed to output a near-optimal policy after polynomially many queries to the generative model. Furthermore, we note that the computational complexity of our algorithm exactly matches its sample complexity up to a factor of the number of actions, which is the cost of sampling from the softmax policies. To see why this is the case, note that for each sample drawn from the simulator, the initial-state distribution and the softmax policy, the algorithm performs a constant number of elementary operations. Finally, we once again stress that our algorithm outputs a globally valid softmax policy that is compactly represented by a $d$-dimensional parameter vector. Thus, to our knowledge, our method is the first global planning method that produces a simple output while being provably efficient both statistically and computationally under a linear MDP assumption (and even relaxed versions thereof).

\section{Analysis}\label{sec:analysis}
This section presents the main components of the proof of our main result, Theorem~\ref{thm:main}. The analysis relies on the definition of a quantity we call the \emph{dynamic duality gap} associated with the iterates of the algorithm, defined with respect to any \emph{primal comparator} $(\lambda^*,u^*)$ and \emph{dual comparator sequence} $\theta^*_{1:T} = (\theta_1^*,\dots,\theta_T^*)$ and $V_{1:T}^* = (V_1^*,\dots,V_T^*)$ as
\[
\GG_T\pa{\lambda^*,u^*;\theta^*_{1:T},V^*_{1:T}} = 
\frac{1}{T}\sum_{t=1}^{T} \pa{\LL(\lambda^{*}, u^{*}; \theta_{t}, V_{t}) - \LL(\lambda_{t},u_{t}; \theta^{*}_{t}, V^{*}_{t})}.
\]
The dynamic duality gap is closely related to the classic notion of duality gap considered in the saddle-point-optimization literature, with the key difference being that the dual comparator is not a static point $(\theta^*,V^*)$, but is rather a sequence of comparator points. Similarly to how the duality gap can be written as the sum of the average regrets of two concurrent regret minimization methods for the primal and dual method, the dynamic duality gap can be written as the sum of the average regret of the primal method and the \emph{dynamic} regret of the dual method. In our analysis below, we relate the dynamic duality gap to the expected suboptimality of the policy $\piout$ produced by our algorithm, and show how the dynamic duality gap itself can be bounded.

Our first lemma shows that the dynamic duality gap evaluated at an appropriately selected comparator sequence can be exactly related to the quantity of our main interest:
\begin{lemma}\label{lem:duality-to-suboptimality}
	Let $\mu^*$ denote the occupancy measure of an optimal policy and $\lambda^*=\B\transpose \mu^*$. Also, let $V^*_t = V^{\pi_t}$ and  
	$\theta^*_t = \argmin_{\theta\in\bb{d}{\thbound}} \infnorm{Q^{\pi_t} - \Phi\theta}$. Then,
	\[
	\GG_T\pa{\lambda^*,\mu^*;\theta^*_{1:T},V^*_{1:T}} \ge \EET{\iprod{\mu^* - \mu^{\piout}}{r}} - 2\EET{\varepsilon_{\piout}} - 2\IBE - \thbound\iprod{\mu^*}{\coreeps}.
	\]
\end{lemma}
Notably, when $\coreeps = \vec{0}$, $\IBE = 0$ and $\varepsilon_{\pi_t} = 0$ hold for all $t$, the claim holds with equality. The proof of this result draws inspiration from \citet{cheng2020reduction}, who first introduced the idea of making use of an adaptively chosen comparator point to reduce duality-gap guarantees to policy suboptimality guarantees in their Proposition~4 (at least to our knowledge). The idea of using a dynamic comparator sequence is new to our analysis and allows us to output a simple softmax policy that is compactly represented by a linearly parametrized Q-function. Below, we prove the lemma for the special case where all approximation errors are zero, and we relegate the slightly more complicated proof for the general case to Appendix~\ref{app:duality-to-suboptimality}.

\paragraph{Proof for zero approximation error.}
Suppose that $\IBE = 0$, $\varepsilon_{\pi_t} = 0$ holds for all $t$ and $\coredelta = \vec{0}$.
Let $Q_t^* = \Phi\theta^*_t = Q^{\pi_t}$ and $\theta'_t$ be such that  $\Phi\theta_t' = r + \gamma PV_t - Q_t$. We start by rewriting each term in the definition of the dynamic duality gap. First, we note that
\begin{align*}
	\LL(\lambda^{*}, u^{*}; \theta_{t}, V_{t}) &= \iprod{\B\transpose\mu^{*}}{\U\pa{r + \gamma PV_{t} - Q_{t}}} +(1 - \gamma)\iprod{\nu_{0}}{V_{t}} + \iprod{\mu^{*}}{Q_{t} - EV_{t}}
	\\
	&= \iprod{\U\transpose\B\transpose\mu^{*}}{r + \gamma PV_{t} - Q_{t}} + \iprod{\mu^{*}}{Q_{t} - \gamma PV_{t}}
	\\
	&= \iprod{\U\transpose\B\transpose\mu^{*}}{\Phi\theta'_t} - \iprod{\mu^{*}}{\Phi\theta'_t - r} = \iprod{\mu^*}{r},
\end{align*}
where in the last line we have used Assumption~\ref{ass:Core_SA_approx} with $\coredelta=0$ that implies $\B\U\Phi = \Phi$.

On the other hand, we have
\begin{align*}
	\LL(\lambda_{t},u_{t}; \theta^{*}_{t}, V^{*}_{t}) 
	&= 
	\iprod{\lambda_{t}}{\U \pa{r + \gamma PV^{\pi_t} - Q^{\pi_t}}} + (1 - \gamma)\iprod{\nu_{0}}{V^{\pi_t}}
	+ \iprod{u_{t}}{Q^{\pi_t} - EV^{\pi_t}}
	\\
	&= (1 - \gamma)\iprod{\nu_{0}}{V^{\pi_t}} =\iprod{\mu^{\pi_t}}{r},
\end{align*}
where the last step follows from the definitions of the value function and the discounted occupancy measure, and the previous step from the fact that the value functions satisfy the Bellman equations $Q^{\pi_{t}} = r + \gamma PV^{\pi_t}$, and that $\iprod{u_t}{EV^{\pi_t}} = \iprod{u_t}{Q^{\pi_t}}$. Indeed, this last part follows from writing
\[
\iprod{u_t}{Q^{\pi_t}} = \sum_x \nu_t(x) \sum_a \pi_t(a|x) Q^{\pi_t}(x,a) = \sum_x \nu_t(x) V^{\pi_t}(x) = \iprod{u_t}{EV^{\pi_t}},
\]
where we made use of the relation between the action-value function and the value function of $\pi_t$.

Putting the above calculations together, we can conclude that
\begin{align*}
	\GG_T\pa{\lambda^*,\mu^*;\theta^*_{1:T},V^*_{1:T}}
	&= \frac{1}{T}\sum_{t=1}^T \iprod{\mu^{*} - \mu^{\pi_{t}}}{r} = \EET{\iprod{\mu^{*} - \mu^{\piout}}{r}},
\end{align*}
thus verifying the claim of the lemma.
\hfill\BlackBox

It remains to show that Algorithm~\ref{alg:Main} indeed guarantees that the dynamic duality gap is bounded. This is done in the following lemma that states a bound in terms of the \emph{conditional relative entropy}\footnote{Technically, this quantity is the conditional relative entropy between the occupancy measures $\mu^{\pi^*}$ and $\mu^{\pi_1}$. We stick with the present notation for clarity. Similar quantities have appeared previously in the context of entropy-regularized reinforcement learning algorithms---see, e.g., \citet{NJG17} and \citet{serrano2021logistic}.} between $\pi^*$ and the initial policy $\pi_1$ defined as $\HH{\pi^*}{\pi_1} = \sum_{x} \nu^{\pi^*}(x) \DD{\pi^*(\cdot|x)}{\pi_1(\cdot|x)}$.
\begin{lemma}\label{lem:duality-bound}
	The dynamic duality gap associated with the iterates produced by Algorithm~\ref{alg:Main} satisfies
	\[
	\EE{\GG_T\pa{\lambda^*,u^*;\theta^*_{1:T},V^*_{1:T}}} \le 
	\frac{\DD{\lambda^*}{\lambda_1}}{\eta T} + \frac{\HH{\pi^*}{\pi_1}}{\beta T} + \frac{2 \thbound^2}{\alpha K} + \frac{\eta m^2(1 + 2R\thbound)^2}{2} + \frac{\beta R^2\thbound^2}{2} + 2 \alpha R^2.
	\]
\end{lemma}
The proof is based on decomposing the dynamic duality gap into the sum of the average regret of the primal method and the average dynamic regret of the dual method. These regrets are then bounded via a standard analysis of stochastic mirror descent methods, as well as via a specialized analysis that takes advantage of the implicitly defined updates for the variables that would be otherwise intractable for mirror descent.
\begin{proof}
	We start by rewriting the dynamic duality gap as follows:
	\begin{align*}
		&\GG_T\pa{\lambda^*,u^*;\theta^*_{1:T},V^*_{1:T}} 
		=\frac{1}{T}\sum_{t=1}^{T}	\bpa{\LL(\lambda^{*}, u^{*}; \theta_{t}, V_{t}) - \LL(\lambda_{t},u_{t}; \theta^{*}_{t}, V^{*}_{t})}
		\\
		&\qquad 
		= \frac{1}{T}\underbrace{\sum_{t=1}^{T} \bpa{\LL(\lambda^{*}, u^{*}; \theta_{t}, V_{t}) - \LL(\lambda_{t},u_{t};\theta_{t}, V_{t})}}_{\primalregret_T(\lambda^*,u^*)} 
		+ \frac{1}{T}\underbrace{\sum_{t=1}^{T}	\bpa{\LL(\lambda_{t},u_{t};\theta_{t}, V_{t})- \LL(\lambda_{t},u_{t}; \theta^{*}_{t}, V^{*}_{t})}}_{\dualregret_T(\theta^*_{1:T},V_{1:T}^*)},
	\end{align*}
	where $\primalregret_T$ and $\dualregret_T$ are the regret of the primal method and the dynamic regret of the dual method, respectively. We bound the two terms separately below. Before doing so, it will be useful to introduce the following shorthand notations for the gradient of the Lagrangian with respect to $\lambda$:
	\begin{align*}
		g_{\lambda}(t) &= \nabla_\lambda \LL(\lambda_t,u_t;\theta_t,V_t) = \U\pa{r + \gamma PV_{t} - Q_{t}}.
	\end{align*}
	
	First, we rewrite the primal regret as 
	\begin{align*}
		\primalregret_T(\lambda^*,u^*) 
		&= \sum_{t=1}^T \bpa{\iprod{\lambda^{*} -\lambda_{t}}{\U \pa{r + \gamma PV_{t} - Q_{t}}} + \iprod{u^{*} - u_t}{Q_{t} - EV_{t} }}
		\\
		&= \sum_{t=1}^T \bpa{\iprod{\lambda^{*} -\lambda_{t}}{g_\lambda(t)} + \iprod{u^{*}}{Q_{t} - EV_{t}}},
	\end{align*}
	where second step follows from recognizing the expression of $g_\lambda(t)$ and exploiting the properties of our choice of $u_t$ and $V_t$. Indeed, notice that 
	\[
	\iprod{u_t}{Q_t} =  \sum_x \nu_t(x) \sum_a \pi_t(a|x) Q_t(x,a) = \sum_x \nu_t(x) V_t(x) = \iprod{u_t}{EV_t}.
	\]
	
	Controlling the remaining two terms can be achieved by the standard analysis of online stochastic mirror descent \citep{NY83,BT03}. We only provide the high-level arguments here and defer the technical details to Appendix~\ref{app:SMA}. Since the primal updates on $\lambda$ directly use stochastic mirror ascent updates, the standard analysis applies and gives a regret bound of 
	\[
	\EE{\sum_{t=1}^T \iprod{\lambda^{*} -\lambda_{t}}{g_\lambda(t)}} \le \frac{\DD{\lambda^*}{\lambda_1}}{\eta} + \frac{\eta Tm^2(1 + 2R\thbound)^2}{2}.
	\]
	To see how a mirror descent analysis applies to the second term in the above bound, note that it can be rewritten as
	\begin{align}
		\nonumber\sum_{t=1}^T \iprod{u^{*}}{Q_{t} - EV_{t}}
		\nonumber&= \sum_x \nu^*(x)\sum_{t=1}^T \sum_{a} \pa{\pi^*(a|x) - \pi_t(a|x)} Q_t(x,a).
	\end{align}
	For each individual $x$, the sum $\sum_{t=1}^T \sum_{a} \pa{\pi^*(a|x) - \pi_t(a|x)} Q_t(x,a)$ can be seen as the regret in a local online learning game with rewards $Q_t(x,\cdot)$ and decision variables $\pi_t(\cdot|x)$. Thus, noting that the update rule for $\pi_t(\cdot|x)$ precisely matches an instance of mirror descent, the standard analysis applies and gives a regret bound of 
	\begin{align}\label{eq:FTL}
		\nonumber\sum_{t=1}^T \iprod{u^{*}}{Q_{t} - EV_{t}}
		&\le \sum_x \nu^*(x) \pa{\frac{\DD{\pi^*(\cdot|x)}{\pi_1(\cdot|x)}}{\beta}+ \frac{\beta TR^{2}\thbound^{2}}{2}} =  \frac{\HH{\pi^*}{\pi_1}}{\beta} + \frac{\beta TR^{2}\thbound^{2}}{2}.
	\end{align}
	To proceed, we rewrite the dynamic regret of the dual method as
	\begin{align*}
		\dualregret_T(\theta^*_{1:T},V_{1:T}^*)
		&= \sum_{t=1}^T \bpa{\iprod{\theta_{t}-\theta^{*}_{t}}{\Phi\transpose u_{t} - \Phi\transpose\U\transpose\lambda_{t}} + \iprod{V_{t} - V^{*}_{t}}{(1 - \gamma)\nu_{0} + \gamma P\transpose\U\transpose\lambda_{t} - E\transpose u_{t}}}\\
		&= \sum_{t=1}^T \iprod{\theta_{t}-\theta^{*}_{t}}{\Phi\transpose u_{t} - \Phi\transpose\U\transpose\lambda_{t}}.
	\end{align*}
	where the last step follows from using the definition of $u_t$ that ensures $E\transpose u_t = \gamma P\transpose \U\transpose\lambda_t + (1-\gamma)\nu_0$, and thus that the second term in the first line is zero. 
	We now notice that the sum is exactly the function that the dual method optimizes within the inner loop of each update, and thus it can be controlled by analyzing the performance of the averaged SGD iterate $\theta_t$. 
	This analysis gives the following bound, stated and proved as Lemma~\ref{lem:SGDbound} in Appendix~\ref{app:SGDbound}:
	\begin{equation}\label{eq:SGDbound}
		\EEt{\iprod{\theta_{t}-\theta^{*}_{t}}{\Phi\transpose u_{t} - \Phi\transpose\U\transpose\lambda_{t}}} \le \frac{\twonorm{\theta_{t-1} - \theta^*_t}^2}{2\alpha K} + 2 \alpha R^2.
	\end{equation}
	Putting all the bounds together concludes the proof.
\end{proof}

\paragraph{Proof of Theorem~\ref{thm:main}.} The proof of the first claim follows from combining Lemmas~\ref{lem:duality-to-suboptimality} and~\ref{lem:duality-bound}. The second part follows from optimizing the constants in the upper bound, and taking into account that the total number of queries to the generative model is $T(K+1)$.
\hfill\BlackBox

\section{Discussion}\label{sec:conc}

In this section, we conclude by discussing the broader context of our result, the 
limitations of our method, and outline some directions for future work.

\paragraph{The role of the core set assumption.} The infamous negative result of \citet{du2020good} states that 
only supposing approximate $Q^{\pi}$-realizability up to an error of $\varepsilon$, there exist environments for which 
any algorithm that finds an $\varepsilon$-optimal policy would require an exponential number of queries to the 
generative model. On the positive side, \cite{lattimore2020learning} showed that this exponential blow-up in the 
feature dimension or effective horizon is avoidable if one only needs the policy 
to be $\Omega\bpa{\varepsilon\sqrt{d}}$-optimal.  Like \cite{shariff2020efficient}, our main result slightly improves 
on this result by assuming some additional structure on the function approximator: access to a small core set of states 
or state-action pairs that can realize all the other possible feature realizations as convex combinations. 
As shown in our Theorem~\ref{thm:main} (and Theorem~3 of \cite{shariff2020efficient}, this condition indeed allows 
sidestepping the negative results mentioned above, allowing us to achieve $O\pa{\varepsilon}$-suboptimality with 
polynomial query complexity. The most burning question of all is of course how to find good core state-action pairs 
satisfying Assumption~\ref{ass:Core_SA_approx}, or if there exist a small set of core state-action pairs under 
reasonable assumptions on the MDP and the feature map. Similar questions arise regarding previous works of 
\citet{zanette2019limiting,shariff2020efficient,wang2021sample}, but we are not aware of a reassuring answer at the 
moment. We note that our definition allows for approximate notions of core state-action pairs, which leaves open the 
possibility of incrementally growing the core set until sufficiently low error is achieved. We are optimistic that 
future work can realize the potential of learning core sets on the fly, based on direct interaction with the MDP rather 
than assuming access to a simulator, and that our results will become more broadly applicable.

\paragraph{The tightness of our bounds.} The bounds stated in Theorem~\ref{thm:main} are unlikely to be tight in 
general, as a comparison with the simple tabular case reveals. Indeed, the tabular case can be modeled in our setting 
with $m = d = |\X||\A|$, and thus the classic lower bounds of \citet{azar2013minimax} imply a query-complexity lower 
bound of $\Omega\pa{\frac{d}{\pa{1-\gamma}^{2}\varepsilon^{2}}}$ for finding an $\varepsilon$-optimal policy. In 
contrast, our bounds are of the considerably higher order 
$O\pa{\frac{d^{2}\log{\pa{d|A|}}}{\pa{1-\gamma}^{4}\varepsilon^{4}}}$. This excess complexity is largely due to the 
double-loop structure of our algorithm, which ensures that the dynamic regret of the dual player remains under control. 
We conjecture that it may be possible to avoid the double-loop structure of our method via more sophisticated 
algorithms incorporating techniques like two-timescale updates \citep{borkar1997stochastic} or optimistic 
mirror descent steps \citep{Kor76,RS13b}. We leave the investigation of such ideas for future work.

\paragraph{Realizability assumptions.} 
Our current analysis requires realizability assumptions that are only mildly weaker than the (arguably rather strong) linear MDP assumption. Whether or not these assumptions can be significantly relaxed is an open question, but we believe that they may be necessary if one aims to output a single compactly parametrized policy as we do. Our algorithm and analysis certainly require this assumption, at least due to the intuition that our method is more ``policy-iteration-like'' than ``value-iteration-like'' due to the extensive policy evaluation steps that it features. Still, the connection with approximate policy iteration methods is weaker than in the case of other algorithms like POLITEX \citep{LABWBS19} or PC-PG \citep{agarwal2020pc}, which leaves some hope for relaxing our realizability assumptions.

\paragraph{Potential extensions.}
The method we develop here is potentially applicable for a much broader range of settings than planning with a 
generative model. In particular, we are confident that our techniques can be applicable for optimistic exploration in 
infinite-horizon MDPs via a combination of tools developed by \citet{NPB20} and \citet{wei2021learning}, to off-policy 
learning \citep{UHJ20,zhan2022offline}, or to imitation learning \citep{kamoutsi2021efficient,viano2022proximal}. We

% Acknowledgments---Will not appear in anonymized version
\acks{This project has received funding from the European Research Council (ERC) under the European Union’s Horizon 2020 research and innovation programme (Grant agreement No.~950180).}

%\nocite{*}
\bibliography{alt2023-sample}

\begin{thebibliography}{50}
\providecommand{\natexlab}[1]{#1}
\providecommand{\url}[1]{\texttt{#1}}
\expandafter\ifx\csname urlstyle\endcsname\relax
  \providecommand{\doi}[1]{doi: #1}\else
  \providecommand{\doi}{doi: \begingroup \urlstyle{rm}\Url}\fi

\bibitem[Agarwal et~al.(2020{\natexlab{a}})Agarwal, Henaff, Kakade, and
  Sun]{agarwal2020pc}
Alekh Agarwal, Mikael Henaff, Sham Kakade, and Wen Sun.
\newblock {PC-PG}: Policy cover directed exploration for provable policy
  gradient learning.
\newblock \emph{Advances in neural information processing systems},
  33:\penalty0 13399--13412, 2020{\natexlab{a}}.

\bibitem[Agarwal et~al.(2020{\natexlab{b}})Agarwal, Kakade, and
  Yang]{agarwal2020model}
Alekh Agarwal, Sham Kakade, and Lin~F. Yang.
\newblock Model-based reinforcement learning with a generative model is minimax
  optimal.
\newblock In \emph{Conference on Learning Theory}, pages 67--83. PMLR,
  2020{\natexlab{b}}.

\bibitem[Antos et~al.(2008)Antos, Szepesv{\'a}ri, and Munos]{ASM08}
Andr{\'a}s Antos, Csaba Szepesv{\'a}ri, and R{\'e}mi Munos.
\newblock Learning near-optimal policies with {Bellman-residual} minimization
  based fitted policy iteration and a single sample path.
\newblock \emph{Machine Learning}, 71\penalty0 (1):\penalty0 89--129, 2008.

\bibitem[Azar et~al.(2013)Azar, Munos, and Kappen]{azar2013minimax}
Mohammad~Gheshlaghi Azar, R{\'e}mi Munos, and Hilbert~J. Kappen.
\newblock Minimax pac bounds on the sample complexity of reinforcement learning
  with a generative model.
\newblock \emph{Machine learning}, 91\penalty0 (3):\penalty0 325--349, 2013.

\bibitem[{Bas-Serrano} and Neu(2020)]{serrano2020faster}
Joan {Bas-Serrano} and Gergely Neu.
\newblock Faster saddle-point optimization for solving large-scale {Markov}
  decision processes.
\newblock In \emph{Learning for Dynamics and Control}, pages 413--423, 2020.

\bibitem[{Bas-Serrano} et~al.(2021){Bas-Serrano}, Curi, Krause, and
  Neu]{serrano2021logistic}
Joan {Bas-Serrano}, Sebastian Curi, Andreas Krause, and Gergely Neu.
\newblock Logistic {Q-Learning}.
\newblock In \emph{International Conference on Artificial Intelligence and
  Statistics}, pages 3610--3618, 2021.

\bibitem[Beck and Teboulle(2003)]{BT03}
Amir Beck and Marc Teboulle.
\newblock Mirror descent and nonlinear projected subgradient methods for convex
  optimization.
\newblock \emph{Operations Research Letters}, 31\penalty0 (3):\penalty0
  167--175, 2003.

\bibitem[Bellman(1966)]{bellman1966dynamic}
Richard Bellman.
\newblock Dynamic programming.
\newblock \emph{Science}, 153\penalty0 (3731):\penalty0 34--37, 1966.

\bibitem[Borkar(1997)]{borkar1997stochastic}
Vivek~S. Borkar.
\newblock Stochastic approximation with two time scales.
\newblock \emph{Systems \& Control Letters}, 29\penalty0 (5):\penalty0
  291--294, 1997.

\bibitem[Cesa-Bianchi and Lugosi(2006)]{CBLu06:book}
Nicol\`o Cesa-Bianchi and G\'abor Lugosi.
\newblock \emph{Prediction, Learning, and Games}.
\newblock Cambridge University Press, New York, NY, USA, 2006.

\bibitem[Chen and Jiang(2019)]{chen2019information}
Jinglin Chen and Nan Jiang.
\newblock Information-theoretic considerations in batch reinforcement learning.
\newblock In \emph{International Conference on Machine Learning}, pages
  1042--1051, 2019.

\bibitem[Chen et~al.(2018)Chen, Li, and Wang]{chen2018scalable}
Yichen Chen, Lihong Li, and Mengdi Wang.
\newblock Scalable bilinear pi learning using state and action features.
\newblock In \emph{International Conference on Machine Learning}, pages
  834--843. PMLR, 2018.

\bibitem[Cheng et~al.(2020)Cheng, Combes, Boots, and
  Gordon]{cheng2020reduction}
Ching-An Cheng, Remi~Tachet Combes, Byron Boots, and Geoff Gordon.
\newblock A reduction from reinforcement learning to no-regret online learning.
\newblock In \emph{International Conference on Artificial Intelligence and
  Statistics}, pages 3514--3524, 2020.

\bibitem[De~Farias and Van~Roy(2003)]{de2003linear}
Daniela~Pucci De~Farias and Benjamin Van~Roy.
\newblock The linear programming approach to approximate dynamic programming.
\newblock \emph{Operations research}, 51\penalty0 (6):\penalty0 850--865, 2003.

\bibitem[De~Farias and Van~Roy(2004)]{de2004constraint}
Daniela~Pucci De~Farias and Benjamin Van~Roy.
\newblock On constraint sampling in the linear programming approach to
  approximate dynamic programming.
\newblock \emph{Mathematics of operations research}, 29\penalty0 (3):\penalty0
  462--478, 2004.

\bibitem[Denardo(1970)]{denardo1970linear}
Eric~V. Denardo.
\newblock On linear programming in a {Markov} decision problem.
\newblock \emph{Management Science}, 16\penalty0 (5):\penalty0 281--288, 1970.

\bibitem[d'Epenoux(1963)]{d1963probabilistic}
Francois d'Epenoux.
\newblock A probabilistic production and inventory problem.
\newblock \emph{Management Science}, 10\penalty0 (1):\penalty0 98--108, 1963.

\bibitem[Du et~al.(2020)Du, Kakade, Wang, and Yang]{du2020good}
Simon~S. Du, Sham~M Kakade, Ruosong Wang, and Lin~F. Yang.
\newblock Is a good representation sufficient for sample efficient
  reinforcement learning?
\newblock In \emph{International Conference on Learning Representations}, 2020.

\bibitem[Jin et~al.(2020)Jin, Yang, Wang, and Jordan]{JYWJ20}
Chi Jin, Zhuoran Yang, Zhaoran Wang, and Michael~I Jordan.
\newblock Provably efficient reinforcement learning with linear function
  approximation.
\newblock In \emph{Conference on Learning Theory}, pages 2137--2143, 2020.

\bibitem[Jin and Sidford(2020)]{jin2020efficiently}
Yujia Jin and Aaron Sidford.
\newblock Efficiently solving {MDPs} with stochastic mirror descent.
\newblock In \emph{International Conference on Machine Learning}, pages
  4890--4900, 2020.

\bibitem[Kamoutsi et~al.(2021)Kamoutsi, Banjac, and
  Lygeros]{kamoutsi2021efficient}
Angeliki Kamoutsi, Goran Banjac, and John Lygeros.
\newblock Efficient performance bounds for primal-dual reinforcement learning
  from demonstrations.
\newblock In \emph{International Conference on Machine Learning}, pages
  5257--5268, 2021.

\bibitem[Korpelevich(1976)]{Kor76}
G.~M. Korpelevich.
\newblock The extragradient method for finding saddle points and other
  problems.
\newblock \emph{Matecon}, 12:\penalty0 747--756, 1976.

\bibitem[Lakshminarayanan et~al.(2017)Lakshminarayanan, Bhatnagar, and
  Szepesv{\'a}ri]{lakshminarayanan2017linearly}
Chandrashekar Lakshminarayanan, Shalabh Bhatnagar, and Csaba Szepesv{\'a}ri.
\newblock A linearly relaxed approximate linear program for {Markov} decision
  processes.
\newblock \emph{IEEE Transactions on Automatic control}, 63\penalty0
  (4):\penalty0 1185--1191, 2017.

\bibitem[Lattimore et~al.(2020)Lattimore, Szepesvari, and
  Weisz]{lattimore2020learning}
Tor Lattimore, Csaba Szepesvari, and Gellert Weisz.
\newblock Learning with good feature representations in bandits and in rl with
  a generative model.
\newblock In \emph{International Conference on Machine Learning}, pages
  5662--5670. PMLR, 2020.

\bibitem[Lazic et~al.(2019)Lazic, Abbasi-Yadkori, Bhatia, Weisz, Bartlett, and
  Szepesv\'ari]{LABWBS19}
Nevena Lazic, Yasin Abbasi-Yadkori, Kush Bhatia, Gellert Weisz, Peter Bartlett,
  and {\text{Cs}}aba Szepesv\'ari.
\newblock {POLITEX}: Regret bounds for policy iteration using expert
  prediction.
\newblock In \emph{International Conference on Machine Learning}, pages
  3692--3702, 2019.

\bibitem[Manne(1960)]{manne1960linear}
Alan~S. Manne.
\newblock Linear programming and sequential decisions.
\newblock \emph{Management Science}, 6\penalty0 (3):\penalty0 259--267, 1960.

\bibitem[Mehta and Meyn(2009)]{mehta2009q}
Prashant Mehta and Sean Meyn.
\newblock Q-learning and {Pontryagin's} minimum principle.
\newblock In \emph{Proceedings of the 48h IEEE Conference on Decision and
  Control (CDC) held jointly with 2009 28th Chinese Control Conference}, pages
  3598--3605. IEEE, 2009.

\bibitem[Nemirovski and Yudin(1983)]{NY83}
Arkadi Nemirovski and David Yudin.
\newblock \emph{Problem Complexity and Method Efficiency in Optimization}.
\newblock Wiley Interscience, 1983.

\bibitem[Neu and Pike-Burke(2020)]{NPB20}
Gergely Neu and Ciara Pike-Burke.
\newblock A unifying view of optimism in episodic reinforcement learning.
\newblock In \emph{Advances in Neural Information Processing Systems}, pages
  1392--1403, 2020.

\bibitem[Neu et~al.(2017)Neu, Jonsson, and G{\'o}mez]{NJG17}
Gergely Neu, Anders Jonsson, and Vicen{\c{c}} G{\'o}mez.
\newblock A unified view of entropy-regularized {Markov} decision processes.
\newblock \emph{arXiv preprint arXiv:1705.07798}, 2017.

\bibitem[Petrik and Zilberstein(2009)]{petrik2009constraint}
Marek Petrik and Shlomo Zilberstein.
\newblock Constraint relaxation in approximate linear programs.
\newblock In \emph{Proceedings of the 26th Annual International Conference on
  Machine Learning}, pages 809--816, 2009.

\bibitem[Puterman(2014)]{puterman2014markov}
Martin~L. Puterman.
\newblock \emph{Markov decision processes: discrete stochastic dynamic
  programming}.
\newblock John Wiley \& Sons, 2014.

\bibitem[Rakhlin and Sridharan(2013)]{RS13b}
Alexander Rakhlin and Karthik Sridharan.
\newblock Optimization, learning, and games with predictable sequences.
\newblock In \emph{Advances in Neural Information Processing Systems}, pages
  3066--3074, 2013.

\bibitem[Schweitzer and Seidmann(1985)]{schweitzer1985generalized}
Paul~J. Schweitzer and Abraham Seidmann.
\newblock Generalized polynomial approximations in {Markovian} decision
  processes.
\newblock \emph{Journal of mathematical analysis and applications},
  110\penalty0 (2):\penalty0 568--582, 1985.

\bibitem[Shariff and Szepesv{\'a}ri(2020)]{shariff2020efficient}
Roshan Shariff and Csaba Szepesv{\'a}ri.
\newblock Efficient planning in large {MDPs} with weak linear function
  approximation.
\newblock \emph{arXiv preprint arXiv:2007.06184}, 2020.

\bibitem[Sidford et~al.(2018)Sidford, Wang, Wu, Yang, and Ye]{sidford2018near}
Aaron Sidford, Mengdi Wang, Xian Wu, Lin Yang, and Yinyu Ye.
\newblock Near-optimal time and sample complexities for solving {Markov}
  decision processes with a generative model.
\newblock \emph{Advances in Neural Information Processing Systems}, 31, 2018.

\bibitem[Sutton and Barto(2018)]{sutton2018reinforcement}
Richard~S. Sutton and Andrew~G. Barto.
\newblock \emph{Reinforcement learning: An introduction}.
\newblock MIT press, 2018.

\bibitem[Tiapkin and Gasnikov(2022)]{tiapkin2022primal}
Daniil Tiapkin and Alexander Gasnikov.
\newblock Primal-dual stochastic mirror descent for {MDPs}.
\newblock In \emph{International Conference on Artificial Intelligence and
  Statistics}, pages 9723--9740. PMLR, 2022.

\bibitem[Uehara et~al.(2020)Uehara, Huang, and Jiang]{UHJ20}
Masatoshi Uehara, Jiawei Huang, and Nan Jiang.
\newblock Minimax weight and {Q}-function learning for off-policy evaluation.
\newblock In \emph{{ICML}}, volume 119 of \emph{Proceedings of Machine Learning
  Research}, pages 9659--9668, 2020.

\bibitem[Viano et~al.(2022)Viano, Kamoutsi, Neu, Krawczuk, and
  Cevher]{viano2022proximal}
Luca Viano, Angeliki Kamoutsi, Gergely Neu, Igor Krawczuk, and Volkan Cevher.
\newblock Proximal point imitation learning.
\newblock \emph{arXiv preprint arXiv:2209.10968}, 2022.

\bibitem[Wang et~al.(2021)Wang, Yan, and Fan]{wang2021sample}
Bingyan Wang, Yuling Yan, and Jianqing Fan.
\newblock Sample-efficient reinforcement learning for linearly-parameterized
  {MDPs} with a generative model.
\newblock \emph{Advances in Neural Information Processing Systems}, 34, 2021.

\bibitem[Wang(2017)]{wang2017primal}
Mengdi Wang.
\newblock Primal-dual pi-learning: Sample complexity and sublinear run time for
  ergodic {Markov} decision problems.
\newblock \emph{arXiv preprint arXiv:1710.06100}, 2017.

\bibitem[Wang and Chen(2016)]{wang2016online}
Mengdi Wang and Yichen Chen.
\newblock An online primal-dual method for discounted {Markov} decision
  processes.
\newblock In \emph{2016 IEEE 55th Conference on Decision and Control (CDC)},
  pages 4516--4521, 2016.

\bibitem[Wei et~al.(2021)Wei, Jahromi, Luo, and Jain]{wei2021learning}
Chen-Yu Wei, Mehdi~Jafarnia Jahromi, Haipeng Luo, and Rahul Jain.
\newblock Learning infinite-horizon average-reward {MDPs} with linear function
  approximation.
\newblock In \emph{International Conference on Artificial Intelligence and
  Statistics}, pages 3007--3015. PMLR, 2021.

\bibitem[Yang and Wang(2019)]{yang2019sample}
Lin Yang and Mengdi Wang.
\newblock Sample-optimal parametric {Q-learning} using linearly additive
  features.
\newblock In \emph{International Conference on Machine Learning}, pages
  6995--7004, 2019.

\bibitem[Yin et~al.(2022)Yin, Hao, Abbasi-Yadkori, Lazi{\'c}, and
  Szepesv{\'a}ri]{yin2022efficient}
Dong Yin, Botao Hao, Yasin Abbasi-Yadkori, Nevena Lazi{\'c}, and Csaba
  Szepesv{\'a}ri.
\newblock Efficient local planning with linear function approximation.
\newblock In \emph{International Conference on Algorithmic Learning Theory},
  pages 1165--1192, 2022.

\bibitem[Zanette et~al.(2019)Zanette, Lazaric, Kochenderfer, and
  Brunskill]{zanette2019limiting}
Andrea Zanette, Alessandro Lazaric, Mykel~J Kochenderfer, and Emma Brunskill.
\newblock Limiting extrapolation in linear approximate value iteration.
\newblock \emph{Advances in Neural Information Processing Systems}, 32, 2019.

\bibitem[Zanette et~al.(2020)Zanette, Lazaric, Kochenderfer, and
  Brunskill]{zanette2020learning}
Andrea Zanette, Alessandro Lazaric, Mykel Kochenderfer, and Emma Brunskill.
\newblock Learning near optimal policies with low inherent {Bellman} error.
\newblock In \emph{International Conference on Machine Learning}, pages
  10978--10989, 2020.

\bibitem[Zhan et~al.(2022)Zhan, Huang, Huang, Jiang, and Lee]{zhan2022offline}
Wenhao Zhan, Baihe Huang, Audrey Huang, Nan Jiang, and Jason Lee.
\newblock Offline reinforcement learning with realizability and single-policy
  concentrability.
\newblock In \emph{Conference on Learning Theory}, pages 2730--2775. PMLR,
  2022.

\bibitem[Zinkevich(2003)]{Zin03}
Martin Zinkevich.
\newblock Online convex programming and generalized infinitesimal gradient
  ascent.
\newblock In \emph{Proceedings of the Twentieth International Conference on
  Machine Learning (ICML)}, 2003.

\end{thebibliography}

\newpage
\appendix
\section{Missing proofs for Section~\ref{sec:RALP}}
\subsection{Analysis of the Relaxed LP for Linear MDPs}\label{app:linMDP-realizability}
Here we show that the optimal solutions of the original linear program are also optimal in the relaxed LP when there exist $W\in\Rn^{d\times X}$, $\vartheta\in\Rn^{d}$ such that $P = \Phi W$, $r = \Phi \vartheta$ and $\coredelta = \vec{0}$. For ease of reference, we denote by $\M$ and $\Tilde{\M}_{\Phi}$ the feasible sets of the original LP and the relaxed LP respectively, and we let $\wt{\M}_{\Phi}\bigr|_u$ denote the set of distributions $u$ such that $(\lambda,u)\in\wt{\M}_{\Phi}$ holds for some $\lambda$. Recall that $\M$ is the set of valid occupancy measures.

For the primal relaxed LP, using $P = \Phi W$ and $\Phi=\B\U\Phi$ it is easy to check that for all $(\lambda,u)\in \Tilde{\M}_{\Phi}$, $u$ is a valid occupancy measure satisfying $E\transpose u=  (1 - \gamma)\nu_{0} + \gamma P\transpose u$. Furthermore, for $u\in\M$, choose $\lambda=\B\transpose u$ so that $\Phi\transpose \U\transpose\lambda= \Phi\transpose u$ holds, which implies feasibility as $(\lambda,u)\in\Tilde{\M}_{\Phi}$. Therefore, we have $\M\subseteq \wt{\M}_{\Phi}\bigr|_u$. Let $(\lambda^{*}, u^{*}) = \argmax\{\langle\lambda\,,\U r\rangle| E\transpose u=  (1 - \gamma)\nu_{0} + \gamma P\transpose \U\transpose\lambda, \Phi\transpose \U\transpose\lambda= \Phi\transpose u, \lambda\in\real_+^{m}, u\in\real^{XA}_+\}$. We already know by the relation $\M\subseteq \wt{\M}_{\Phi}\bigr|_u$ that $\iprod{\B\transpose\mu^{*}}{\U r}\leq \iprod{\lambda^{*}}{\U r}$ but since $r=\Phi\vartheta$ and $\Phi=\B\U\Phi$, we also have
$$\iprod{\lambda^{*}}{\U r}
= \iprod{\lambda^{*}}{\U\Phi\varphi}
= \iprod{\Phi\transpose\U\transpose\lambda^{*}}{\vartheta}
= \iprod{\Phi\transpose u^{*}}{\vartheta}
= \iprod{u^{*}}{\Phi\vartheta}
= \iprod{u^{*}}{r}.$$
Also
$$\iprod{\B\transpose\mu^{*}}{\U r}
= \iprod{\mu^{*}}{\B\U r}
= \iprod{\mu^{*}}{\B\U\Phi\vartheta}
= \iprod{\mu^{*}}{\Phi\vartheta}
= \iprod{\mu^{*}}{r}$$
That is, $\iprod{\mu^{*}}{r}\leq \iprod{u^{*}}{r}$. But since $u^{*}$ is a valid occupancy measure, we must have $\iprod{\mu^{*}}{r} = \iprod{u^{*}}{r}$ due to optimality of $\mu^*$ on the space of occupancy measures. Therefore, any optimal occupancy measure is feasible and optimal in the primal relaxed ALP when the MDP is linear and the core state action assumption is satisfied exactly.

For the dual, notice that linearity of the transition probability and reward functions imply that there exists $\theta^{*}\in\real^{d}$ such that $Q^{*} = \Phi\theta^{*}$. Hence, feasibility of $(V^{*},\theta^{*})$ follows by definition so we only need to check optimality in the relaxed dual LP. Let $T^{*}_{V}:\real^{X}\rightarrow\real^{X}$ denote the Bellman optimality operator for state value functions. We know that for any feasible $(V,\theta)$, we have $\U Q_{\theta}\geq \U[r + \gamma PV]$, which by monotonicity of $\B$ implies $\B\U Q_{\theta}\geq \B\U[r + \gamma PV]$. Recalling that both $Q_\theta$ and $r + \gamma PV$ are linear\footnote{This is the only point where the proof uses the assumption that the transition model is linear. Also, when the reward vector is known one can easily prove linearity of rewards in the feature map. Notice that this implies that the result holds under the much weaker condition of $\Phi$ being complete under the Bellman optimality operator.\label{fn:realizability}} in $\Phi$ and $\B\U\Phi = \Phi$, this further implies $Q_{\theta}\geq r + \gamma PV$.
Then by the first constraint, we have $EV\geq Q_{\theta}\geq r + \gamma PV$, which implies $V\geq M[r + \gamma PV]= T^{*}_{V}V$. Then, by monotonicity of the Bellman optimality operator with $V^{*}$ as its unique fixed point, for any feasible $(V,\theta)$ of the relaxed dual LP, we can repeat the above argument and obtain 
\[
 V\geq T_V^*V \geq \pa{T_V^*}^2 V \geq \dots \geq V^{*}.
\]
In a similar way using the Bellman optimality operator acting on $Q$-functions, we can also show that $Q_{\theta}\geq Q_{\theta^{*}}$. Therefore, $(V^{*},\theta^{*})$ is optimal in the relaxed LP, and these solutions are unique when $\nu_{0}$ has full support.

\section{Missing proofs for Section~\ref{sec:analysis}}
\subsection{Properties of the gradient estimators}\label{app:estimates}
Here we state a pair of results that establish some useful properties of the gradient estimators calculated by Algorithm~\ref{alg:Main}. The first one concerns the gradient estimators used by the primal method.
\begin{lemma}\label{lem:glambda}
 The gradient estimator $\Tilde{g}_{\lambda}(t)$ satisfies $\EEc{\Tilde{g}_{\lambda}(t)}{\F_{t-1},\theta_{t}} = \nabla_{\lambda}\LL(\lambda_{t},u_{t}; \theta_{t}, V_{t})$ and $\infnorm{\Tilde{g}_{\lambda}(t)} \le m(1 + (1+\gamma)R\thbound)$.
\end{lemma}
\begin{proof}
The first claim follows from the following straightforward calculation:
	\begin{align*}
		\EEc{\Tilde{g}_{\lambda}(t)}{\F_{t-1},\theta_{t}}
		&= \EEc{m[r(x_{t},a_{t})  + \gamma V_{t}(y_{t}) - Q_{t}(x_{t},a_{t})]\vec{e}_{(x_{t},a_{t})}}{\F_{t-1},\theta_{t}}\\
		&= m\sum_{(x,a)\in\Tilde{\Z}}\dfrac{1}{m}\sum_{y\in\X}P(y|x,a)[r(x,a)  + \gamma V_{t}(y) - Q_{t}(x,a)]\vec{e}_{(x,a)}\\
		&= \sum_{(x,a)\in\Tilde{\Z}}\big[r(x,a)  + \gamma\sum_{y\in\X}P(y|x,a) V_{t}(y) - Q_{t}(x,a)\big]\vec{e}_{(x,a)}\\
		&= \U[r + \gamma PV_{t} - Q_{t}] = \nabla_{\lambda}\LL(\lambda_{t},u_{t}; \theta_{t}, V_{t}).
	\end{align*}
	For the second half, recall that $r\in[0,1]$ and $\infnorm{V_{t}} \le \infnorm{Q_{t}}\leq R\thbound$, so that we can write
	\begin{align*}
		\left\|\Tilde{g}_{\lambda}(t)\right\|_{\infty}
		&= \left\|m\pa{r(x_{t},a_{t})  + \gamma V_{t}(y_{t}) - Q_{t}(x_{t},a_{t})}\vec{e}_{(x_{t},a_{t})}\right\|_{\infty}\\
		&= m\left\|\pa{r(x_{t},a_{t})  + \gamma V_{t}(y_{t}) - Q_{t}(x_{t},a_{t})}\vec{e}_{(x_{t},a_{t})}\right\|_{\infty}\\
		&= m\pa{r(x_{t},a_{t})  + \gamma V_{t}(y_{t}) - Q_{t}(x_{t},a_{t})} \leq m(1+(1+\gamma)R\thbound).
	\end{align*}
	This completes the proof.
\end{proof}
The second result concerns the gradient estimators used by the dual method for updating $\theta$.
\begin{lemma}\label{lem:gtheta}
 The gradient estimator $\Tilde{g}_{\theta}(t,i)$ satisfies\footnote{The notation $\EEti{\cdot}$ is as defined in the proof of Lemma~\ref{lem:duality-bound}.} $\EEti{\Tilde{g}_{\theta}(t,i)} = \nabla_{\theta}\LL(\lambda_{t},u_{t}; \theta_{t-1}, V_{t-1})$ and $\twonorm{\Tilde{g}_{\theta}(t,i)} \le 2R$.
\end{lemma}
\begin{proof}
We first write
	\begin{align*}
		\EEti{\Tilde{g}_{\theta}(t,i)}
		&= \EEti{(1-\gamma)\varphi(x_{0,t}^{(i)},a_{0,t}^{(i)}) + \gamma\varphi(\ol{x}_{t}^{(i)},\ol{a}_{t}^{(i)})- \varphi(x_{t}^{(i)},a_{t}^{(i)})}\\
		&= \Phi\transpose\EEc{(1-\gamma)\vec{e}_{(x_{0,t}^{(i)},a_{0,t}^{(i)})} + \gamma\vec{e}_{(\ol{x}_{t}^{(i)},\ol{a}_{t}^{(i)})}- \vec{e}_{(x_{t}^{(i)},a_{t}^{(i)})}}{\F_{t,i}}\\
		&= (1-\gamma)\Phi\transpose\sum_{(x,a)\in\Z}\nu_{0}(x)\pi_{t}(a|x)\vec{e}_{(x,a)}\\
		&\quad
		+ \gamma\Phi\transpose\sum_{(x,a)\in\Z}\pi_{t}(a|x)\pa{\sum_{(x',a')\in\Tilde{\Z}}P(x|x',a')\lambda_{t}(x',a')}\vec{e}_{(x,a)}\\
		&\quad
		- \Phi\transpose\sum_{(x,a)\in\Tilde{\Z}}\lambda_{t}(x,a)\vec{e}_{(x,a)}\\
		&=\Phi\transpose\bpa{\pa{(1-\gamma)\nu_{0} + \gamma P\transpose\U\transpose\lambda_{t}}\circ \pi_{t} -  \U\transpose\lambda_{t}}\\
		&= \Phi\transpose u_{t} - \Phi\transpose\U\transpose\lambda_{t}\\
		&= \nabla_{\theta}\LL(\lambda_{t},u_{t}; \theta_{t-1}, V_{t-1})
	\end{align*}
	Also, using that $\|\varphi(x,a)\|_{2} \leq R$ for all $(x,a)\in\Z$, we get
	\begin{align*}
		\|\Tilde{g}_{\theta}(t,i)\|_{2}
		&= \|(1-\gamma)\varphi(x_{0,t}^{(i)},a_{0,t}^{(i)}) + \gamma\varphi(\ol{x}_{t}^{(i)},\ol{a}_{t}^{(i)})- \varphi(x_{t}^{(i)},a_{t}^{(i)})\|_{2}\\
		&\leq (1-\gamma)\|\varphi(x_{0,t}^{(i)},a_{0,t}^{(i)})\|_{2} + \gamma\|\varphi(\ol{x}_{t}^{(i)},\ol{a}_{t}^{(i)})\|_{2} + \|\varphi(x_{t}^{(i)},a_{t}^{(i)})\|_{2}\\
		&\leq (1-\gamma)R + \gamma R + R = 2R,
	\end{align*}
	thus concluding the proof.
\end{proof}

\subsection{Stochastic mirror ascent analysis}\label{app:SMA}
	This section presents the final steps in the primal regret analysis deferred from the proof of Lemma~\ref{lem:duality-bound}. In particular, we will show an upper bound on 
	\[
	 \primalregret_T(\lambda^*,u^*) = \sum_{t=1}^T \bpa{\iprod{\lambda^{*} -\lambda_{t}}{g_\lambda(t)} + \iprod{u^{*}}{Q_{t} - EV_{t}}}.
	\]
	The first sum above is clearly the regret of online stochastic mirror ascent, and the second term can be written as an average of local regrets in each state $x$ as 
	\begin{align}\label{eq:local_decomposition}
		\nonumber\sum_{t=1}^T \iprod{u^{*}}{Q_{t} - EV_{t}}
		\nonumber&= \sum_x \nu^*(x)\sum_{t=1}^T \sum_{a} \pa{\pi^*(a|x) - \pi_t(a|x)} Q_t(x,a).
	\end{align}
	The following lemma will allow us to treat the two sums in a unified manner.
	\begin{lemma}\label{lem:OMD}
		Given $\omega_{1}\in\Delta_{N}$, define the sequence of vectors $\omega_2,\dots,\omega_{n+1}$ and $g_1,g_2,\dots,g_n$ in $\real^N$ via the following recursion. For each $k$, let $\wt{g}_k$ satisfy $\EEcc{\wt{g}_k}{\F_{k-1}} = g_k$ and let
		\begin{equation*}
			\omega_{k+1,i} = \frac{\omega_{k,i} e^{\tau \wt{g}_{k,i}}}{\sum_{j=1}^{N}\omega_{k,j}e^{\tau\Tilde{g}_{k,j}}}\qquad\qquad\text{for }i=1,\cdots,N.
		\end{equation*}
		Supposing that $\infnorm{\wt{g}_k}^{2} \le G^{2}$ holds, the following bound is satisfied for any $\omega^{*}\in\Delta_{N}$,
		\[
		\EE{\sum_{k=1}^n \iprod{\omega^* - \omega_k}{g_k}} \le \frac{\DD{\omega^*}{\omega_1}}{\tau} + \frac{\tau n G^2}{2}.
		\]
	\end{lemma}
	We provide the standard proof below. To put this result to work, note that the recursion defining $\lambda$ clearly matches the above requirements using the unbiased gradient estimators $\wt{g}_\lambda(t)$ that satisfy $\infnorm{\wt{g}_\lambda(t)}\le m(1 + 2 R\thbound)$ (as shown Lemma~\ref{lem:glambda} in Appendix~\ref{app:estimates}). Thus, we can apply Lemma~\ref{lem:OMD} to obtain the bound
	\[
	\EE{\sum_{t=1}^T \iprod{\lambda^{*} -\lambda_{t}}{g_\lambda(t)}} \le \frac{\DD{\lambda^*}{\lambda_1}}{\eta} + \frac{\eta Tm^2(1 + 2R\thbound)^2}{2}.
	\]
	In order to use the result again to bound the other term, notice that the sequence of policies $\pi_t(\cdot|x)$ and action-value functions $Q_t(x,\cdot)$ are indeed generated by a recursion of the form required by Lemma~\ref{lem:OMD}. Further noting that $\infnorm{Q_t} \le R\thbound$, we can apply the lemma to each $x$ individually to obtain
	\begin{align}
		\nonumber\sum_{t=1}^T \iprod{u^{*}}{Q_{t} - EV_{t}}
		&\le \sum_x \nu^*(x) \pa{\frac{\DD{\pi^*(\cdot|x)}{\pi_1(\cdot|x)}}{\beta}+ \frac{\beta TR^{2}\thbound^{2}}{2}} =  \frac{\HH{\pi^*}{\pi_1}}{\beta} + \frac{\beta TR^{2}\thbound^{2}}{2}.
	\end{align}
	This proves the claims in the proof of Lemma~\ref{lem:duality-bound}.
	
\paragraph{Proof of Lemma \ref{lem:OMD}}
Using that $\omega^{*}$, $\{\omega_{k}\}_{k=1}^{n} \in\Delta_{N}$, the proof follows from straightforward calculations involving the relative entropies between $\omega^*$ and iterates $\omega_k$ and $\omega_{k+1}$:
\begin{align*}
	\DD{\omega^{*}}{\omega_{k+1}}
	&= \sum_{i=1}^{N}\omega^{*}_{i}\log \frac{\omega^{*}_{i}}{\omega_{k+1,i}}\\
	&= \sum_{i=1}^{N}\omega^{*}_{i}\log \frac{\omega^{*}_{i}}{\omega_{k,i}} - \sum_{i=1}^{N}\omega^{*}_{i}\log \frac{\omega_{k+1,i}}{\omega_{k,i}}\\
	&= \DD{\omega^{*}}{\omega_{k}}
	- \sum_{i=1}^{N}\omega^{*}_{i}\log \frac{\omega_{k,i}e^{\tau\Tilde{g}_{k,i}}}{\omega_{k,i}\sum_{j=1}^{N}\omega_{k,j}e^{\tau\Tilde{g}_{k,j}}}\\
	&= \DD{\omega^{*}}{\omega_{k}}
	- \sum_{i=1}^{N}\omega^{*}_{i}\log \frac{e^{\tau\Tilde{g}_{k,i}}}{\sum_{j=1}^{N}\omega_{k,j}e^{\tau\Tilde{g}_{k,j}}}\\
	&= \DD{\omega^{*}}{\omega_{k}}
	- \iprod{\omega^{*}}{\tau\wt{g}_{k}}
	+ \log\sum_{i=1}^{N}\omega_{k,i}e^{\tau\Tilde{g}_{k,i}}\\
	&= \DD{\omega^{*}}{\omega_{k}} - \iprod{\omega^{*} - \omega_{k}}{\tau\wt{g}_{k}}
	- \iprod{\omega_{k}}{\tau\wt{g}_{k}}
	+ \log\sum_{i=1}^{N}\omega_{k,i}e^{\tau\Tilde{g}_{k,i}}\\
	&\leq \DD{\omega^{*}}{\omega_{k}} - \iprod{\omega^{*} - \omega_{k}}{\tau\wt{g}_{k}}
	+ \frac{\tau^2}{2} \infnorm{\wt{g}_{k}}^2,
\end{align*}
where the last inequality follows from the condition on $\tau$ and Hoeffding's lemma (cf.~Lemma~A.1 in \citealp{CBLu06:book}). Taking conditional expectations on both sides, using that $\EEcc{\wt{g}_k}{\F_{k-1}} = g_k$, and reordering the terms, we get
\[
\iprod{\omega^{*} - \omega_{k}}{g_k} \le 
\frac{\DD{\omega^{*}}{\omega_{k}} - \EEs{\DD{\omega^{*}}{\omega_{k+1}}}{k}}{\tau}
+ \frac{\tau}{2}\EEs{\infnorm{\wt{g}_{k}}^2}{k}.
\]
After summing up for all $k$ and taking marginal expectations on both sides, we obtain
\[
\EE{\sum_{k=1}^n \iprod{\omega^{*} - \omega_{k}}{g_{k}}} \le 
\frac{\DD{\omega^{*}}{\omega_{1}} - \EE{\DD{\omega^{*}}{\omega_{T+1}}}}{\tau}
+ \frac{\tau}{2}\sum_{k=1}^n \EE{\infnorm{\wt{g}_{k}}^2}.
\]
Upper bounding the negative divergence term by $0$ and using that $\infnorm{\wt{g}_k}^{2} \le G^{2}$ holds for $k=1,\cdots,n$ concludes the proof.
\hfill$\blacksquare$

\subsection{Stochastic gradient descent analysis}\label{app:SGDbound}
In this section we bound the optimization error incurred in each dual update to the parameter vector $\theta$.  Combined with the facts that $\norm{\theta_{t-1} - \theta^*_t}\le 2\thbound$ and using the gradient bounds stated in Lemma~\ref{lem:gtheta}, the following lemma proves the bound of Equation~\eqref{eq:SGDbound} appearing in the main text.
\begin{lemma}\label{lem:SGDbound}
 For any $\theta^*_t\in\bb{b}{d}$, the iterate $\theta_t$ satisfies the bound
	\[
	\iprod{\theta_{t}-\theta^{*}_{t}}{\Phi\transpose u_{t} - \Phi\transpose\U\transpose\lambda_{t}} \le \frac{\twonorm{\theta_{t-1} - \theta^*_t}^2}{2\alpha K} + \frac{\alpha}{2K} \sum_{i=1}^K \EE{\twonorm{\wt{g}_\theta(t,i)}^2}.
	\]
\end{lemma}
\begin{proof}
The proof follows from classic calculations that can be familiar from references like \citet{NY83} or \citet{Zin03}. We will use $\F_{t,i}$ to denote the sigma-algebra generated by the interaction history up to iteration $i-1$ within the inner loop for updating the dual variables in round $t$. We will use the shorthand notation $\EEti{\cdot} = \EEcc{\cdot}{\F_{t,i}}$ to signify expectations conditioned on the history up to this point. Using the definition of $\theta_{t}^{(i+1)}$ in Algorithm~\ref{alg:Main}, the following bound holds for any $\theta^{*}_{t}\in \bb{b}{d}$, $i$ and $t$:
		\begin{align*}
			\|\theta_{t}^{(i+1)} - \theta^{*}_{t}\|_{2}^{2}
			&= \|\Pi_{\bb{b}{d}}(\theta_{t}^{(i)} - \alpha \Tilde{g}_{\theta}(t,i)) - \theta^{*}_{t}\|_{2}^{2}\\
			&\leq \|\theta_{t}^{(i)} - \theta^{*}_{t} - \alpha \Tilde{g}_{\theta}(t,i)\|_{2}^{2}\\
			&= \|\theta_{t}^{(i)} - \theta^{*}_{t}\|_{2}^{2}
			-2\alpha\biprod{\theta_{t}^{(i)} - \theta^{*}_{t}}{\Tilde{g}_{\theta}(t,i)}
			+\alpha^{2}\|\Tilde{g}_{\theta}(t,i)\|_{2}^{2},
		\end{align*}
		where the inequality follows from the fact that the projection operator is a nonexpansion with respect to the Euclidean norm and the rest follows from straightforward manipulations. After reordering, taking conditional expectations on both sides, and observing that $\EEs{\wt{g}_{\theta}(t,i)}{t,i} = \Phi\transpose u_{t} - \Phi\transpose\U\transpose\lambda_{t}$, we obtain
		\[
		\biprod{\theta_{t}^{(i)} - \theta^{*}_{t}}{\Phi\transpose u_{t} - \Phi\transpose\U\transpose\lambda_{t}}
		\leq \frac{\|\theta_{t}^{(i)} - \theta^{*}_{t}\|_{2}^{2} - \EEs{\|\theta_{t}^{(i+1)} - \theta^{*}_{t}\|_{2}^{2}}{t,i}}{2\alpha} + \frac{\alpha}{2}\EEs{\|\Tilde{g}_{\theta}(t,i)\|_{2}^{2}}{t,i}
		\]
		Then, summing over $i$ and taking the marginal expectation gives
		\[
			\EE{\sum_{i=1}^K \biprod{\theta_{t}^{(i)} - \theta^{*}_{t}}{\Phi\transpose u_{t} - \Phi\transpose\U\transpose\lambda_{t}}} \le \frac{\|\theta_{t}^{(1)} - \theta^{*}_{t}\|_{2}^{2}}{2\alpha} + \frac{\alpha}{2}\sum_{i=1}^K\EE{\|\Tilde{g}_{\theta}(t,i)\|_{2}^{2}}
		\]
Dividing both sides by $K$ and recalling the definition of $\theta_t$ concludes the proof.
\end{proof}

\subsection{The full proof of Lemma~\ref{lem:duality-to-suboptimality}}\label{app:duality-to-suboptimality}
\begin{proof}
	Let $Q_t^* = \Phi\theta^*_t$ and $\delta^*_t = Q^{\pi_t} - Q^*_t$, and also $\theta'_t = \argmin_{\theta\in\B(\thbound)} \iprod{(\U\transpose\B\transpose - I)\mu^*}{r + \gamma PV_t - \Phi\theta}$ and $\delta_t = r + \gamma PV_t - \Phi\theta'_t$. We start by rewriting each term in the definition of the dynamic duality gap. First, we consider the following term:
	\begin{align*}
		\LL(\lambda^{*}, u^{*}; \theta_{t}, V_{t}) &= \iprod{\B\transpose\mu^{*}}{\U\pa{r + \gamma PV_{t} - Q_{t}}} +(1 - \gamma)\iprod{\nu_{0}}{V_{t}} + \iprod{\mu^{*}}{Q_{t} - EV_{t}}
		\\
		&= \iprod{\U\transpose\B\transpose\mu^{*}}{r + \gamma PV_{t} - Q_{t}} + \iprod{\mu^{*}}{Q_{t} - \gamma PV_{t}}
		\\
		&= \iprod{\U\transpose\B\transpose\mu^{*}}{\Phi\pa{\theta_t - \theta_t'} + \delta_t} - \iprod{\mu^{*}}{\Phi\pa{\theta_t - \theta_t'} + \delta_t } + \iprod{\mu^*}{r}
		\\
		&= \iprod{\mu^*}{r} + \iprod{\pa{\U\transpose\B\transpose - I}\mu^{*}}{\delta_t} +\iprod{\mu^*}{\coredelta(\theta_t-\theta_t')},
	\end{align*}
	where in the last line we have used Assumption~\ref{ass:Core_SA_approx} that implies $\B\U\Phi = \Phi-\coredelta$.
	The last term can be bounded as $\iprod{\mu^*}{\coredelta(\theta_t-\theta_t')} \le \iprod{\mu^*}{\coreeps}\twonorm{\theta_t-\theta_t'} \le 2 \thbound\iprod{\mu^*}{\coreeps}$. Finally, the term in the middle can be bounded as follows, using the definition of the Inherent Bellman Error (Definition~\ref{def:IBE}):
	\begin{equation}
		\begin{split}\label{eq:IBE_refined}
			\iprod{\pa{\U\transpose\B\transpose - I}\mu^{*}}{\delta_t} &= \iprod{\pa{\U\transpose\B\transpose - I}\mu^{*}}{r + \gamma PV_t - \Phi\theta'_t} 
			\\
			&= \iprod{\pa{\U\transpose\B\transpose - I}\mu^{*}}{\BE_{\pi_t}(\theta'_t, \theta_t)}
			\\
			&= \min_{\theta\in\bb{\thbound}{d}} \iprod{\pa{\U\transpose\B\transpose - I}\mu^{*}}{\BE_{\pi_t}(\theta, \theta_t)} \le 2\IBE.
		\end{split}
	\end{equation}
	
	On the other hand, we have
	\begin{align*}
		&\LL(\lambda_{t},u_{t}; \theta^{*}_{t}, V^{*}_{t}) 
		= 
		\iprod{\lambda_{t}}{\U \pa{r + \gamma PV^{*}_t - Q_{t}^{*}}} + (1 - \gamma)\iprod{\nu_{0}}{V^{*}_t}
		+ \iprod{u_{t}}{Q_{t}^{*} - EV^{*}_t}
		\\
		&\quad= \iprod{\lambda_{t}}{\U \pa{r + \gamma PV^{\pi_t} - Q^{\pi_t} + \delta^*_t}} 
		+ (1 - \gamma)\iprod{\nu_{0}}{V^{\pi_t}}
		+ \iprod{u_{t}}{Q^{\pi_t} - EV^{\pi_t} - \delta^*_t}
		\\
		&\quad=\iprod{\mu^{\pi_t}}{r} + \iprod{\U\transpose\lambda_t - u_t}{\delta^*_t},
	\end{align*}
	where the last step follows from observing that $(1 - \gamma)\iprod{\nu_{0}}{V^{\pi_t}} = \iprod{\mu^{\pi_t}}{r}$ by the definition of the value function and the occupancy measure, and also that $\iprod{u_t}{EV^{\pi_t}} = \iprod{u_t}{Q^{\pi_t}}$ and that the value functions satisfy the Bellman optimality equations $Q^{\pi_{t}} = r + \gamma PV^{\pi_t}$. In order to bound the last term, notice that by the definition of the Q-evaluation error in Definition~\ref{def:QAE}, we have $\infnorm{\delta_t^*}\le \varepsilon_{\pi_t}$, and thus we have $\iprod{\U\transpose\lambda_t - u_t}{\delta^*_t}\le 2 \varepsilon_{\pi_t}$.
	
	Putting the above bounds together, we can conclude that the dynamic duality gap at the chosen set of comparators can be bounded as
	\begin{align*}
		\GG_T\pa{\lambda^*,\mu^*;\theta^*_{1:T},V^*_{1:T}}
		&\ge \frac{1}{T}\sum_{t=1}^T \pa{\iprod{\mu^{*} - \mu^{\pi_{t}}}{r} - 2\varepsilon_{\pi_t}} - 2\IBE - 2\thbound \iprod{\mu^*}{\coreeps}
		\\
		&= \EE{\iprod{\mu^{*} - \mu^{\piout}}{r}} - 2\EE{\varepsilon_{\piout}} - 2\IBE - 2\thbound \iprod{\mu^*}{\coreeps}.
	\end{align*}
	Reordering the terms completes the proof.
\end{proof}

\end{document}